\documentclass[lettersize,journal]{IEEEtran}
\usepackage{amsmath,amsfonts,bm}
\usepackage{algorithm}
\usepackage{array}
\usepackage{subcaption}
\usepackage{textcomp}
\usepackage{stfloats}
\usepackage{url}
\usepackage{verbatim}
\usepackage{graphicx}
\usepackage{cite}
\usepackage[capitalise]{cleveref}
\crefname{assumption}{Assumption}{Assumptions}
\crefname{problem}{Problem}{Problems}
\usepackage{selinput}\SelectInputMappings{adieresis={ä},germandbls={ß}}
\usepackage{epsfig}
\usepackage{times}
\usepackage{float}
\usepackage{tikz,pgfplots}
\usepackage[english]{babel}
\usepackage{algpseudocode}
\usepackage{verbatim}
\usepackage{makecell}
\usepackage{multirow}
\usepackage{xcolor}
\usetikzlibrary{arrows,shapes,backgrounds,patterns,fadings,matrix,arrows,calc,
	intersections,decorations.markings,
	positioning,arrows.meta}
\usepgfplotslibrary{fillbetween}
\usepgfplotslibrary{statistics}
\pgfplotsset{width=5\columnwidth /5, compat = 1.13,
	height = 60\columnwidth /100, grid= major,
	legend cell align = left, ticklabel style = {font=\scriptsize},
	every axis label/.append style={font=\small},
	legend style = {font={\scriptsize}},title style={yshift=-7pt, font = \small} }

\newtheorem{assumption}{\bf{Assumption}}
\newtheorem{theorem}{\bf{Theorem}}
\newtheorem{corollary}{\bf{Corollary}}
\newtheorem{lemma}{\bf{Lemma}}
\newtheorem{remark}{\bf{Remark}}

\newtheorem{property}{\bf{Property}}
\newtheorem{problem}{\bf{Problem}}

\crefname@preamble{assumption}{Assumption}{Assumptions}
\crefname@preamble{lemma}{Lemma}{Lemmas}
\crefname@preamble{property}{Property}{Properties}

\begin{document}
	
	\title{
		Event-triggered Control and Online Learning for Networked Systems under Computational Delays
	}
	
	\author{
		Xiaobing Dai$^{1}$, Armin Lederer$^{2}$,~\IEEEmembership{Member, IEEE}, Zewen Yang$^{3*}$,~\IEEEmembership{Member, IEEE}, Sihua Zhang$^{1}$, Lu Wan$^{4}$ \\
		Yang Tang$^{5}$,~\IEEEmembership{Fellow, IEEE}, Sandra Hirche$^{1}$,~\IEEEmembership{Fellow, IEEE}
		\thanks{
			$^*$Corresponding author.
		}
		\thanks{
			$^{1}$Xiaobing Dai, Sihua Zhang and Sandra Hirche are with the Chair of Information-oriented Control, School of Computation, Information and Technology, Technical University of Munich, Germany (email: xiaobing.dai, sihua.zhang, hirche@tum.de).
		}
		\thanks{
			$^{2}$Armin Lederer is with the Learning and Adaptive Systems Group, Department of Computer Science, ETH Zurich, Switzerland (e-mail: armin.lederer@inf.ethz.ch).
		}
		\thanks{
			$^{3}$Zewen Yang is with the Chair of Robotics and Systems Intelligence, Munich Institute of Robotics and Machine Intelligence, Technical University of Munich (email: zewen.yang@tum.de).
		}
		\thanks{
			$^{4}$Lu Wan is with Information Systems in the Built Environment, Eindhoven University of Technology, Eindhoven, Netherlands. (email: l.wan@tue.nl).
		}
		\thanks{
			$^{5}$Yang Tang is with the Key Laboratory of Advanced Control and Optimization for Chemical Processes, Ministry of Education, East China
			University of Science and Technology, Shanghai, China (e-mail: tangtany@gmail.com).
		}
	}
	
	\markboth{Journal of \LaTeX\ Class Files,~Vol.~14, No.~8, August~2021}%
	{Shell \MakeLowercase{\textit{et al.}}: A Sample Article Using IEEEtran.cls for IEEE Journals}
	
	
	\maketitle
	
	\begin{abstract}
		Online learning-based control is a promising approach to control uncertain systems, where unknown components are identified during operation to improve control performance.
		However, resource-intensive online learning algorithms introduce non-negligible computational delays, especially when executed on systems with limited local computational resources.
		To mitigate this, an in-network online learning-based control structure is employed by deploying the learning-based controller on a remote computation node and connecting it via a communication channel.
		In this paper, control performance guarantee is first established by deriving tracking error bound for the in-network control architecture, while accounting for computational delays.
		The derived tracking error bound allows for diverse communication and computation strategies under a specific condition, including time-/event-triggered mechanisms.
		Additionally, the trade-off between communication and computation performances is shown for a given desired control performance.
		Furthermore, to enhance the efficiency in both communication and computation, an efficient control framework with an asynchronous event-triggered mechanism in both control and online learning is devised under the existence of computational delay.
		The proposed event-triggered strategy is proven to achieve the same control performance as time-triggered scenario while excluding Zeno behavior.
		Finally, we derive an explicit expression of the proposed event-trigger condition for exponentially stabilizable systems, and demonstrate its effectiveness through simulations.
	\end{abstract}
	
	
	\section{Introduction}
	\label{section_introduction}
	
	With increasing system complexity and high demand of adaptability across various scenarios, control of uncertain dynamics becomes challenging in many applications such as surface vehicles \cite{yang2025safe}, soft robotics \cite{gao2021quasi} and human-centric systems \cite{dai2023fast}.
	To address this challenge, supervised machine learning \cite{burkart2021survey} is increasingly employed to infer unknown components in systems from collected data sets.
	The learned data-driven model is then integrated in model-based controllers, such as computed torque control \cite{nuno2011synchronization} and model predictive control \cite{prajapat2024towards}, to improve control performance \cite{yang2016neural} and enhance robustness \cite{na2020output}.
	
	The strong relationship between the control performance and the inference accuracy of machine learning models is shown in e.g., \cite{lederer2020training}.
	To improve inference performance, online learning is a promising strategy by updating the prediction model using newly collected data \cite{hoi2021online} during the operation.
	However, inference and online model updates impose substantial computational burden and induce significant data storage requirements \cite{liu2020gaussian}.
	High inference and online learning-time compute \cite{zaremba2025trading} prevents the deployment on systems with limited computational resources.
	One option to alleviate the delays from computationally intensive machine learning is to employ an in-network online learning-based control structure, which deploys the learning-based controller on a remote node connected to the local control system via a communication channel \cite{lederer2023gaussian} as shown in \cref{figure_networked_architecture}.
	In this paper, we analyze the performance of such in-network online learning-based control under non-negligible computational delay on the remote node.
	Moreover, a control and online learning strategy is devised via event-triggered mechanism, to achieve high efficiencies in both communication and computation.
	
	\begin{figure}[t]
		\centering
		\includegraphics[width=0.48\textwidth]{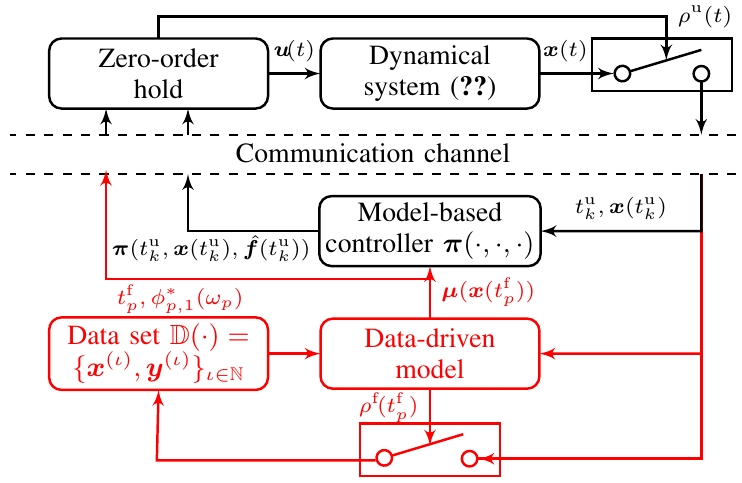}
		\caption{ 
			Block diagram for in-network online learning-based control with event-triggered control and online data collection mechanism.
			The black blocks and lines imply fast computation and communication, while the computation and communication in {\color{red} red blocks and lines} are slow.
		}
		\label{figure_networked_architecture}
	\end{figure}
	
	\subsection{Related Works}
	
	To reduce the computational burden on the local control system, an in-network control framework is adopted \cite{zhang2019networked, rosolia2020multi}, where the computationally intensive controller is executed remotely and communicates with the local plant over a network.
	For such in-network structure, the delay effect induced by limited communication power is well studied with conventional model-based controller using e.g., computed torque control \cite{yang2025safe}, model predictive control \cite{chen2025distributed} and control barrier function \cite{zhang2024learning}.
	For better adaptability of unknown dynamics control, computationally intensive machine learning tools are introduced forming an in-network learning-based control framework, where delays is dominantly caused by computation rather than communication \cite{dai2023can}.
	Unlike the commonly studied communication delays, which are typically modeled as continuous data transmission subject to constant or time-varying shifts \cite{michiels2005stabilization}, computational delays instead introduce a blocking interval during which further operations are suspended while computation is being executed \cite{dai2023can}.
	Only few works \cite{dai2023can, zhao2025probabilistic} investigate the effect from machine learning induced computational delays on in-network control performance, which is specified for feedback linearized systems.
	However, they assume continuous state transmission for feedback control, which is impractical in an in-network setting with constrained communication resources \cite{hu2007sampled}.
	In this study, a comprehensive analyses is provided for in-network learning-based control under computational delays, accounting for practical limitations in communication power.
	Furthermore, the efficiencies of communication and computation are essential for in-network learning-based control, which saves resources and energy.
	To this end, event-triggered control and online learning are promising to achieve high efficiency \cite{heemels2012introduction, umlauft2019feedback}, while preserving a desired control performance.
	
	Event-triggered learning \cite{solowjow2020event} is widely studied when unknown nonlinearities are modeled as a linear combination of known nonlinear features \cite{zou2024novel}, where the trigger law determines the time instances for weight updates.
	However, the specific structure of the unknown component restricts their extension to more general systems.
	Moreover, the absence of theoretical performance guarantees for prediction accuracy with practically unknown feature functions hinders their application in safe control \cite{zhang2022adaptive}.
	Due to the modeling flexibility for continuous unknown functions \cite{williams2006gaussian} and the existence of theoretical prediction error bounds \cite{srinivas2012information, hashimoto2022learning}, event-triggered Gaussian process (GP) based control is increasingly applied to achieve high learning and data efficiency \cite{umlauft2019feedback, hashimoto2020learning}, particularly in safety-critical applications \cite{huang2024learning, zhang2024learning}. 
	Furthermore, the computational delay is also considered in the event-trigger design in \cite{dai2023can, zhao2025probabilistic}, leading to a delay-aware event-triggered online learning strategy.
	However, these event-triggered GP methods typically assume continuous control and evaluation of GP outputs, which is possible only when the learning-based control is deployed on the local system, necessitating significant local computational resources.
	
	The combination of event-triggered control and learning for unknown systems is explored in \cite{yang2025safe}, where the updates for control inputs and machine learning model are triggered simultaneously.
	This synchronous updating approach leads to more frequent model and control updates, resulting in high computational and communication demands.
	While asynchronous dual trigger strategies are investigated in \cite{zhang2021adaptive,zhang2023event}, they merely focus on event-triggered communication within the sensor-controller and controller-actuator channels with known system model.
	Therefore, their methods cannot be directly extended to include event-triggered online learning.
	For asynchronous event-triggered control and online learning, a common issue always exists to determine ``whether to update the controller or machine learning model when control performance is not sufficiently good''.
	Existing dual trigger strategies struggle to address this question, because they assume that the controller and learning model always receive the same information.
	In practice, the difference in computational demands between the machine learning and the controller should be considered.
	Due to the generally significantly higher computation times for the machine learning algorithm, it can only receive new information after the previous computation task \cite{lederer2023gaussian}, while the controller can be considered to continuously receive new data at any time. 
	This information gap between controller and the machine learning model motivates the aforementioned question, and necessitates the design of a new asynchronous event-triggered control and online learning framework explicitly taking computational delays into account.
	To achieve high efficiency of both communication and computation, we focus on the joint design of an event-triggered mechanism for online learning and control within such a fast-slow structure.
	
	\subsection{Contributions}
	
	In this paper, we investigate the networked online learning-based control architecture as shown in \cref{figure_networked_architecture} for a tracking task of general unknown nonlinear systems, where the existence of computational delay from machine learning is considered.
	The unknown components of the systems are predicted using Gaussian process regression, which provides the prediction performance quantification.
	First, the control performance in terms of tracking error bound is analyzed under computational delays and a specific condition for control and online learning strategy, where a trade-off between communication and computation performances is shown for a given desired control performance.
	Moreover, an exemplar control and online learning strategy is provided via time-triggered mechanism to achieve the specific condition in performance analysis.

	Furthermore, to achieve the desired control performance with higher efficiencies in both communication and computation, an asynchronous event-triggered control and online learning framework is devised under computational delay, for which Zeno behavior is excluded.
	Additionally, explicit closed-form expressions for the trigger conditions are presented for exponentially stabilizable systems as in \cite{umlauft2019feedback, jiao2022backstepping}, to demonstrate the practical implementability of our proposed framework.
	The effectiveness of the proposed control and online learning framework is demonstrated via simulations in comparison to different control and online learning strategies.
	
	\subsection{Notation and Structure}
	
	The set of real numbers are denoted as $\mathbb{R}$, and let $\mathbb{R}_+$/$\mathbb{R}_{0,+}$ represents positive/non-negative reals $(0, \infty)$/$[0, \infty)$. 
	Natural numbers are symbolized by $\mathbb{N}$ and $\mathbb{N}_+$ without zero. 
	The $N \times N$ identity matrix is $\bm{I}_N$ with $N \in \mathbb{N}_+$. 
	The trace of a matrix $\bm{A}$ is written as $\mathrm{tr}(\bm{A})$.
	
	The remaining content of this paper is structured as follows. 
	\cref{section_problem} presents the problem setting. 
	In \cref{section_control_performance}, the online learning-based control performance is analyzed.
	The event-triggered control and online learning framework is proposed in \cref{section_double_trigger}.
	\cref{section_simulation} shows the numerical simulations, followed by the conclusion in \cref{section_conclusion}. 	
	
	\section{Problem Setting and Preliminaries}
	\label{section_problem}
	\subsection{System Description and Control Objective}
	\label{subsection_problem}
	
	In this paper, a nonlinear system is considered with the form
	\begin{align} \label{eqn_general_system}
		\dot{\bm{x}}(t) = \underbrace{\bm{h}(\bm{x}(t), \bm{u}(t))}_{\text{known part}} +  \underbrace{\bm{f}(\bm{x}(t))}_{\text{unknown part}},
	\end{align}
	where $\bm{x} \in \mathbb{X} \subset \mathbb{R}^{n}$ with compact domain $\mathbb{X}$ and $\bm{u} \in \mathbb{R}^m$ denote the system states and control inputs with $n, m \in \mathbb{N}_+$, respectively.
	The continuous function $\bm{h}(\cdot,\cdot): \mathbb{X} \times \mathbb{R}^m \to \mathbb{R}^{n}$ encodes all prior knowledge of the system, and therefore is assumed to be known.
	The continuous function $\bm{f}(\cdot) \!=\! [f_1(\cdot), \!\cdots\!, f_n(\cdot)]^T:\! \mathbb{X} \!\to\! \mathbb{R}^n$ with $f_i(\cdot):\! \mathbb{X} \!\to\! \mathbb{R}$ for $i \!=\! 1, \!\cdots\!, n$ includes all uncertainties such as unmodeled components and environmental perturbations, and is considered as unknown.
	
	\begin{remark}
		The system is described in continuous time, which covers a large range of practical systems including mechanical systems, electrical systems, chemical processes and power systems. 
		The framework is directly relevant to digitally implemented control systems. 
		Many physical processes are inherently continuous-time systems, whereas discrete-time models arise from sampling and may not capture inter-sampling behavior. 
		The continuous-time formulation enables rigorous analysis of stability and event-triggered mechanisms, including the effect of delays between events. 
		For implementation on digital platforms, the control input can be applied via a zero-order hold and the triggering condition can be checked at sampling instants, resulting in a sampled-data realization of the proposed method.
	\end{remark}
	
	Moreover, system \eqref{eqn_general_system} satisfies the following assumption.
	
	\begin{assumption} \label{assumption_Lh}
		The known function $\bm{h}(\cdot, \cdot)$ is Lipschitz w.r.t the second input with a known Lipschitz constant $L_h \!\in\! \mathbb{R}_{0,+}$, which means $\| \bm{h}(\bm{x}, \bm{u}_1) - \bm{h}(\bm{x}, \bm{u}_2) \| \le L_h \| \bm{u}_1 - \bm{u}_2 \|$ for any $\bm{x} \in \mathbb{X}$ and any $\bm{u}_1, \bm{u}_2 \in \mathbb{R}^m$.
	\end{assumption}

	Lipschitz continuity is a common assumption in nonlinear control, as Lipschitz continuity of the closed-loop dynamics guarantees the existence of a unique solution of the dynamical system \cite{khalil2015nonlinear}.
	Furthermore, the knowledge of Lipschitz constant $L_h$ is straightforward since $\bm{h}(\cdot, \cdot)$ is known.
	
	\begin{remark}
		In this paper, the process and measurement noise on the system state $\bm{x}$ are neglected, allowing us to focus more on the event-triggered learning and control design.
		For systems with process and measurement noise, we refer to \cite{buisson2021joint,bin2020model}, and the extension of the proposed framework to such cases is left for future work.
	\end{remark}

	The control objective is to track a desired known continuous state trajectory $\bm{x}_r(t)$ with $\bm{x}_r(\cdot): \mathbb{R}_{0,\!+} \!\to\! \mathbb{X}$.
	To this end, a control law $\bm{\pi}(\cdot, \cdot, \cdot): \mathbb{R}_{0,+} \times \mathbb{X} \times \mathbb{R}^n \to \mathbb{R}^m$ is proposed, such that $\bm{\pi}(t, \bm{x}, \hat{\bm{f}}(t))$ employs time $t$, system states $\bm{x}$ and the predicted function $\hat{\bm{f}}(\cdot): \mathbb{R}_{0,+} \to \mathbb{R}^n$ of $\bm{f}(\bm{x}(\cdot))$.
	Moreover, the law $\bm{\pi}(\cdot, \cdot, \cdot)$ is assumed to have the following property.
	
	\begin{assumption} \label{assumption_u}
		There exist well-defined known Lipschitz constants $L_{\pi,t}, L_{\pi,x}, L_{\pi,f} \!\in\! \mathbb{R}_{0,+}$, such that
		\begin{align}
			\| \bm{\pi}(t_1,\bm{x}_1,\bm{f}_{x,1}) - \bm{\pi}(t_2,\bm{x}_1,\bm{f}_{x,1}) \| &\le L_{\pi,t} |t_1 - t_2|, \nonumber \\
			\| \bm{\pi}(t_1,\bm{x}_1,\bm{f}_{x,1}) - \bm{\pi}(t_1,\bm{x}_2,\bm{f}_{x,1}) \| &\le L_{\pi,x} \| \bm{x}_1 - \bm{x}_2 \|, \\
			\| \bm{\pi}(t_1,\bm{x}_1,\bm{f}_{x,1}) - \bm{\pi}(t_1,\bm{x}_1,\bm{f}_{x,2}) \| &\le L_{\pi,f} \| \bm{f}_{x,1} - \bm{f}_{x,2} \| \nonumber
		\end{align}
		for all $t_1, t_2 \in \mathbb{R}_{0,+}$, $\forall \bm{x}_1, \bm{x}_2 \in \mathbb{X}$ and $\forall \bm{f}_{x,1}, \bm{f}_{x,2} \in \mathbb{R}^n$.
		Moreover, the upper bound of $\| \bm{\pi}(t, \bm{x}, \bm{0} ) \|$ is known as $\bar{\pi} \in \mathbb{R}_{0,+}$, i.e., $\| \bm{\pi}(t, \bm{x}, \bm{0} ) \| \le \bar{\pi}$ for all $t \in \mathbb{R}_{0,+}$ and $\forall \bm{x} \in \mathbb{X}$.
	\end{assumption}
	
	Lipschitz continuity in \cref{assumption_u} is commonly found in continuous nonlinear control laws \cite{khalil2015nonlinear}, such as feedback linearization \cite{khalil2015nonlinear} and computed torque control \cite{yang2025safe}.
	Moreover, the bounded control input by $\bar{\pi}$ holds in many real world applications, making \cref{assumption_u} practically non-restrictive.

	Additionally, the law $\bm{\pi}(\cdot,\cdot,\cdot)$ is chosen, such that the controlled system achieves uniform global asymptotic stability with exact system model, i.e., $\bm{u}(t) = \bm{\pi}(t, \bm{x}(t), \bm{f}(\bm{x}(t)))$.
	This emulation-based property is formulated as follows.
	
	\begin{assumption} \label{assumption_u_performance}
		For state $\bm{x}(\cdot)$ generated by system \eqref{eqn_general_system} with arbitrary control input $\bm{u}(\cdot)$, there exists a continuous differentiable function $V(\cdot, \cdot): \mathbb{R}_{0,+} \times \mathbb{R}_{0,+} \to \mathbb{R}_{0,+}$ satisfying\footnotemark
		\begin{align}
			&\alpha_1(e(t)) \le V(t, \bm{x}(t) - \bm{x}_r(t) ) \le \alpha_2(e(t)), \\
			&\dot{V}(t, \bm{x}(t) \!-\! \bm{x}_r(t) ) \le - \alpha( V(t, \bm{x}(t) \!-\! \bm{x}_r(t)) ) + \gamma(\psi(t)),
		\end{align}
		where $e(t) \!=\! \| \bm{x}(t) \!-\! \bm{x}_r(t) \|$, $\psi(t) \!=\! \| \bm{u}(t) \!-\! \bm{\pi}(t, \bm{x}(t), \bm{f}(\bm{x}(t))) \|$ and $\alpha_1(\cdot), \alpha_2(\cdot), \gamma(\cdot) \in \mathcal{K}$, $\alpha(\cdot) \in \mathcal{K}_{\infty}$.
	\end{assumption}
	
	\cref{assumption_u_performance} indicates input-to-state stability, if $\psi(t)$ has a well-defined uniform upper bound.
	Note that the exact value of $\psi(t)$ is unavailable due to the unknown $\bm{f}(\bm{x}(t))$, but its upper bound can be derived as in \cref{section_control_performance}.
	Instead of given specific system and controller structure as in \cite{umlauft2019feedback, jiao2022backstepping}, \cref{assumption_u_performance} only extracts the necessary property of the control law, such that the derived results cover a large range of system classes and nonlinear controllers.
	
	\footnotetext{
	For notational simplicity, denote $V(t) \!=\! V(t, e(t))$ and $\dot{V}(t) \!=\! \dot{V}(t, e(t))$.
	}
	
	To compensate for the unknown part $\bm{f}(\bm{x}(\cdot))$ in \eqref{eqn_general_system}, the estimation $\hat{\bm{f}}(\cdot)$ is applied in $\bm{\pi}(\cdot,\cdot,\cdot)$, which is obtained by the data-driven method introduced in the following subsection.
	
	\subsection{Networked Learning and Control Setting with Delay} 
	\label{subsection_introduction_event_triggered_onlineLearning}
	
	\begin{figure}[t]
		\centering
		\includegraphics[width=0.48\textwidth]{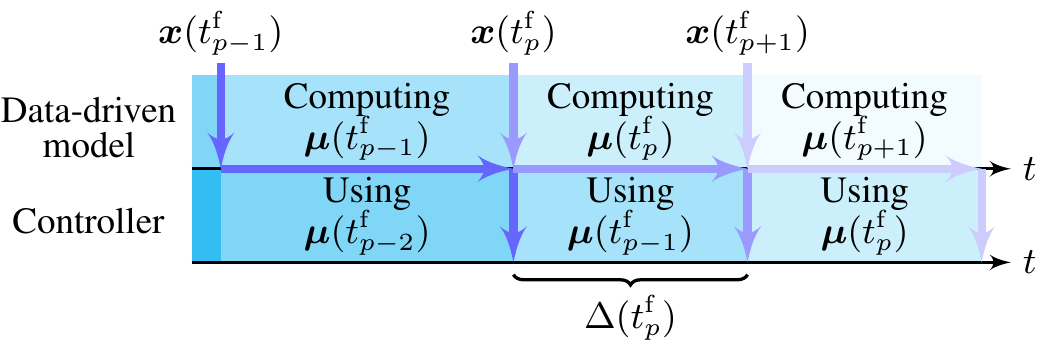}
		\caption{
			Illustration of effects from computational delay.
		}
		\label{figure_computational_delay}
	\end{figure}

	The prediction $\hat{\bm{f}}(t)$ used in $\bm{\pi}(\cdot,\cdot,\cdot)$, serving as the compensation of the unknown component $\bm{f}(\bm{x}(t))$ in \eqref{eqn_general_system}, is obtained through data driven methods from a data set $\mathbb{D}(t)$, which results in a model $\bm{\mu}(\cdot): \mathbb{X} \!\to\! \mathbb{R}^n$ of $\bm{f}(\cdot)$.
	Due to the non-negligible computation time for the evaluation of the data driven model $\bm{\mu}(\cdot)$, the value of $\bm{\mu}(\bm{x}(t))$ is not available at $t$ and therefore cannot be employed as $\hat{\bm{f}}(t)$.
	In practice, a delayed version of the prediction $\bm{\mu}(\bm{x}(\cdot))$ is employed as compensation $\hat{\bm{f}}(\cdot)$ as shown in \cref{figure_computational_delay}, i.e.,
	\begin{align} \label{eqn_hat_f_t_mu}
		\hat{\bm{f}}(t) = \bm{\mu}(\bm{x}(t^{\rm{f}}_p), \mathbb{D}(t^{\rm{f}}_p)), && \forall t \in [t^{\rm{f}}_{p+1}, t^{\rm{f}}_{p+2}),
	\end{align}
	where the time sequence $\{ t^{\rm{f}}_p \}_{p \in \mathbb{N}}$ is defined as
	\begin{align} \label{eqn_time_sequence_tf}
		t^{\rm{f}}_{p+1} = t^{\rm{f}}_p + \Delta(t^{\rm{f}}_p), && \forall p \in \mathbb{N}
	\end{align}
	with $t^{\rm{f}}_0 = 0$.
	The function $\Delta(\cdot): \mathbb{R}_{0,+} \to \mathbb{R}_{0,+}$ denotes the computation time. 
	Specifically $\Delta(t^{\rm{f}}_p)$ is the time required for the machine learning process starting at $t^{\rm{f}}_p$.
	
	\begin{remark}
		Computational delay is an inherent issue in online learning-based control. 
		While faster hardware can reduce latency, the delay introduced by inference and online model updates cannot generally be eliminated, particularly in embedded and mobile platforms. 
		For example, recent online GP acceleration methods require approximately $40$ ms for SkyGP \cite{yang2026streaming} and $200$ ms for LoG-GP \cite{lederer2021gaussian} under a standard laptop configuration \cite{yang2026streaming}. 
		These computation times exceed by orders of magnitude the $1$ ms sampling period of robotic systems operating at $1$ kHz \cite{zhang2020modular}. 
		Hence, computational delay should be explicitly incorporated into the analysis and design of online learning-based control systems.
	\end{remark}
	
	The data set $\mathbb{D}(t^{\rm{f}}_p)$ available at $t^{\rm{f}}_p$, which determines the quality of the data-driven model and affects the prediction accuracy, consists of $N(t^{\rm{f}}_p) = | \mathbb{D}(t^{\rm{f}}_p) |$ data pairs $\{ \bm{x}^{(\iota)}, \bm{y}^{(\iota)} \}$ with $\iota = 1, \cdots, N(t^{\rm{f}}_p)$ satisfying the following assumption.
	
	\begin{assumption} \label{assumption_dataset}
		The data pair $\{ \bm{x}^{(\iota)}, \bm{y}^{(\iota)} \}$ is perturbed by i.i.d sub-Gaussian noise $\bm{w}^{(\iota)}$ with variance proxy $\bm{\Sigma}_w \!=\! \mathrm{diag} (\sigma_{w,1}^2, \!\cdots\!, \sigma_{w,n}^2)$ and $\sigma_{w,i} \!\in\! \mathbb{R}_+$ for all $i \!=\! 1, \!\cdots\!, n$, i.e., $\bm{y}^{(\iota)} \!=\! [y_1^{(\iota)}, \!\cdots\!, y_n^{(\iota)}]^T = \bm{f}(\bm{x}^{(\iota)}) + \bm{w}^{(\iota)}$.
		Moreover, the data pair $\{ \bm{x}(t^{\rm{f}}_p), \bm{y}(t^{\rm{f}}_p) \}$ is available at $t^{\rm{f}}_p$ for all $p \in \mathbb{N}$.
	\end{assumption}
	
	\cref{assumption_dataset} admits noisy measurements $\bm{y}$ of $\bm{f}(\cdot)$, which is usually calculated as $\bm{y} = \dot{\hat{\bm{x}}} - \bm{h}(\bm{x}, \bm{u})$ with the approximated derivative $\dot{\hat{\bm{x}}}$ of $\dot{\bm{x}}$ from e.g., finite difference \cite{yin2023learning}.
	Moreover, it also allows the collection of new data pairs $\{ \bm{x}(t^{\rm{f}}_p), \bm{y}(t^{\rm{f}}_p) \}$ at $t^{\rm{f}}_p$, enabling online learning.
	
	The data set $\mathbb{D}(\cdot)$ of the data-driven methods are crucial for inference accuracy \cite{williams2006gaussian}, which also affects the control performance \cite{lederer2020training}.
	Due to online available data pairs in \cref{assumption_dataset}, it is possible to update the data-driven model during operation by adding new samples into $\mathbb{D}(\cdot)$.
	This data collection process is also called online learning or model update in this paper.
	While online learning is a promising method to enhance prediction accuracy, it requires more computation resources and induces additional processing time \cite{hoi2021online}.
	To analyze the delay effects additionally induced by online learning, each computation time $\Delta(t^{\rm{f}}_p)$ is separated as
	\begin{align}
		\Delta(t^{\rm{f}}_p) = \Delta^{\rm{p}}(t^{\rm{f}}_p) + \Delta^{\rm{u}}(t^{\rm{f}}_p), && \forall p \in \mathbb{N},
	\end{align}
	where $\Delta^{\rm{p}}(\cdot): \mathbb{R}_{0,+} \to \mathbb{R}_{0,+}$ and $\Delta^{\rm{u}}(\cdot): \mathbb{R}_{0,+} \to \mathbb{R}_{0,+}$ return the computation times for \underline{p}rediction of $\bm{\mu}(\bm{x}(\cdot))$ and model \underline{u}pdate related to $\mathbb{D}(\cdot)$, respectively.
	Moreover, the computation times satisfy the following assumption.
	\begin{assumption} \label{assumption_computational_time}
		The computation time $\Delta^{\rm{p}}(\cdot)$ and $\Delta^{\rm{u}}(\cdot)$ are bounded by known constants $\bar{\Delta}^{\rm{p}}, \bar{\Delta}^{\rm{u}} \!\in\! \mathbb{R}_+$, i.e., $\Delta^{\rm{p}}(t^{\rm{f}}_p) \!\le\! \bar{\Delta}^{\rm{p}}$ and $\Delta^{\rm{u}}(t^{\rm{f}}_p) \!\le\! \bar{\Delta}^{\rm{u}}$ for all $p \!\in\! \mathbb{N}$.
	\end{assumption}
	Since the bounded computation times only require the termination of the inference and model update algorithm after a finite number of mathematical operations, \cref{assumption_computational_time} is satisfied by most commonly used data-driven methods.
	Therefore, this assumption poses no severe restriction.
	
	\begin{remark}
		In this paper, we focus on the computational delay induced by online Gaussian process inference and model updating. 
		To keep the notation concise and to highlight the main contributions on event-triggered control and online learning under computational delay, possible communication-induced delays are not explicitly modeled. 
		This simplification is reasonable in the considered setting because the computational delay of online learning can be dominant. 
		For instance, even recently developed real-time online GP methods, such as SkyGP, report computation times of approximately $40$ ms \cite{yang2026streaming}, whereas communication delays in low-latency networked control implementations can be on the order of $1$ ms and may be further reduced by wired communication. 
		Nevertheless, communication effects may become relevant in certain applications, especially in wireless or mobile systems. 
		The joint treatment of computational delay, communication-induced delay, packet losses, and other network-related issues such as cyber-attacks constitutes an important direction for future research.
	\end{remark}
	
	\begin{remark}
		In this work, a deterministic computation time bound is adopted to simplify the notation, allowing the focus to remain on performance analysis and event-trigger design.
		In practice, the computation delays $\Delta^{\mathrm{p}}(\cdot)$ and $\Delta^{\mathrm{u}}(\cdot)$ for prediction and GP model update may occasionally exceed their nominal upper bounds $\bar{\Delta}^{\mathrm{p}}$ and $\bar{\Delta}^{\mathrm{u}}$ due to unexpected computational bottlenecks.
		To capture this, \cref{assumption_computational_time} can be relaxed to a probabilistic form, where the computation delays are modeled as random variables following certain distributions.
		Consequently, the performance analysis should incorporate such violation probability.
	\end{remark}
	
	\begin{remark}
		Under the bounded data set size $\bar{N}$ and rank-one Cholesky updates, GP prediction and online update have complexity $\mathcal{O}(\bar{N}^2)$. 
		Nevertheless, the associated computation-induced delay remains non-negligible for real-time control. 
		Indeed, accelerated online GP methods that retain prediction error guarantees, such as SkyGP \cite{yang2026streaming} and LoG-GP \cite{lederer2021gaussian}, still require computation times of tens of milliseconds or more. 
		This can exceed the allowable time budget in high-frequency robotic systems, e.g., $1$ ms per cycle for a $1000$ Hz controller of Franka. 
		Therefore, it is necessary to explicitly account for computational delays in the learning-based control framework.
	\end{remark}
	
	Considering the high computational load from data-driven methods, a networked architecture for learning-based control is applied, separating the computationally heavy learning-based controller from the plant through communication network as shown in \cref{figure_networked_architecture}.
	For practical digital communication channel, the transmission is discontinuous and only conducted at discrete times $\{ t^{\rm{u}}_k \}_{k \in \mathbb{N}}$, such that the control input received on the local plant with a zero-order hold filter is written as
	\begin{align} \label{eqn_u}
		\bm{u}(t) = \bm{\pi}(t^{\rm{u}}_k, \bm{x}(t^{\rm{u}}_k), \hat{\bm{f}}(t^{\rm{u}}_k)), && \forall t \in [t^{\rm{u}}_k, t^{\rm{u}}_{k+1})
	\end{align}
	for all $k \in \mathbb{N}$ with $t^{\rm{u}}_0 = 0$.
	The time sequence $\{ t^{\rm{u}}_k \}_{k \in \mathbb{N}}$ is determined later from different control strategies, and the transmission interval $t^{\rm{u}}_{k+1} - t^{\rm{u}}_k$ is lower bounded by a strictly positive constant denoted as $\varepsilon \in \mathbb{R}_+$, whose value relies on the power of the communication channel inducing $t^{\rm{u}}_{k+1} - t^{\rm{u}}_k \ge \varepsilon$.

	\subsection{Problem Formulation}
	
	In this paper, we first investigate the performance of in-network online learning-based control, which is characterized by the upper bound of the tracking error $e(\cdot)$ defined in \cref{assumption_u_performance} and is formulated as follows.
	\begin{problem} \label{problem_performance}
		Considering the in-network online learning-based control for unknown system \eqref{eqn_general_system}, which satisfies a specific condition for the applied control and online learning strategy.
		Derive the tracking error bound $\bar{e}(\cdot)$, such that
		\begin{align}
			e(t) \le \bar{e}(t), && \forall t \in \mathbb{R}_{0,+}
		\end{align}
		holds with a high probability and $\bar{e}(\cdot): \mathbb{R}_{0,+} \to \mathbb{R}_{0,+}$.
	\end{problem}
	
	While the sufficient condition for achiving the boundness by $\bar{e}(\cdot)$ allows multiple control and online learning strategies, a smart control and online learning strategy is necessary for high efficiencies in both communication and computation.
	Specifically, the event-triggered mechanism for both control and online learning is devised under computational delay, taking the derived tracking error bound $\bar{e}(t)$ as a baseline performance.
	Such event-trigger determines the times to update control input $\bm{u}$ in \eqref{eqn_general_system} and machine learning model by adding new data points to $\mathbb{D}$ in \eqref{eqn_hat_f_t_mu} asynchronously, whose design problem is formally formulated as follows.
	
	\begin{problem} \label{problem_trigger}
		Consider an unknown system \eqref{eqn_general_system} and let $\bar{e}(\cdot)$ be the desired tracking error bound.
		Design a Zeno-free event-triggered framework for control and online learning under computational delay $\Delta(\cdot)$, such that the tracking error $e(t)$ is uniformly bounded by $\bar{e}(t)$, $\forall t \in \mathbb{R}_{0,+}$ with a high probability.
	\end{problem}
	
	In the following paper, the control performance for \cref{problem_performance} is analyzed in \cref{section_control_performance}, and the event-trigger design addressing \cref{problem_trigger} is shown in \cref{section_double_trigger}.

	\section{In-networked Online Learning-based Control}
	\label{section_control_performance}
	
	In this section, the control performance analysis for in-network online learning-based control under computational delay is analyzed.
	First, Gaussian process regression is introduced in \cref{subsection_GP} with prediction error quantification.
	Then, the control performance in terms of a tracking error bound $\bar{e}(\cdot)$ is determined under a sufficient condition on control and online learning performance in \cref{subsection_control_performance}.
	Next, \cref{subsection_time_trigger} shows an exemplar control and online strategy is shown to satisfy this condition via time-triggered mechanism.
	
	\subsection{Gaussian Process Regression}
	\label{subsection_GP}
	
	Gaussian process regression, regarded as a Bayesian method, is commonly used for unknown function inference by inducing a Gaussian distribution over functions.
	In this subsection, we consider the inference of a scalar unknown function $f_i(\cdot)$ from the $i$-th dimension of $\bm{f}(\cdot)$ with $i = 1, \cdots, n$, whose prior knowledge is characterized by the prior mean function $m_i(\cdot): \mathbb{X} \to \mathbb{R}$ and kernel function $\kappa_i(\cdot, \cdot): \mathbb{X} \times \mathbb{X} \to \mathbb{R}_{0,+}$, i.e., $f_i \sim \mathcal{GP}(m_i, \kappa_i)$.
	The mean function is set as $m_i(\cdot) = 0$ in this paper, since all prior knowledge in \eqref{eqn_general_system} is encoded in $\bm{h}(\cdot)$.
	The kernel function $\kappa_i(\cdot, \cdot)$ reflects the covariance between data pairs satisfying the following assumption.
	
	\begin{assumption} \label{assumption_GP}
		The unknown function $f_i$ belongs to the reproducing kernel Hilbert space (RKHS) $\mathcal{H}_{\kappa_i}$ corresponding to the Lipschitz stationary kernel $\kappa_i(\cdot, \cdot)$ and has a bounded RKHS norm, i.e., $\| f_i \|_{\kappa_i} = \sqrt{\langle f_i, f_i \rangle_{\kappa_i}} \le \Gamma_i$ with $\Gamma_i \in \mathbb{R}^+$ for all $f_i \in \mathcal{H}_{\kappa_i}$.
		The Lipschitz constant for kernel is denoted by $L_{\kappa,i} \in \mathbb{R}_+$ and induced by the Euclidean norm, i.e., $\| \kappa_i(\bm{x}, \bm{x}_1) - \kappa_i(\bm{x}, \bm{x}_2) \| \le L_{\kappa,i} \| \bm{x}_1 - \bm{x}_2 \|$ for all $\bm{x}, \bm{x}_1, \bm{x}_2 \in \mathbb{X}$.
		Moreover, it has $\kappa_i(\bm{x}, \bm{x}) = \sigma_{f,i}^2$ for all $\bm{x} \in \mathbb{X}$.
	\end{assumption}
	
	The RKHS is a large function space including e.g., all analytic functions for the squared exponential kernel \cite{williams2006gaussian}.
	The RKHS norm represents the smoothness of $f_i(\cdot)$ uniquely characterized by the kernel $\kappa_i(\cdot, \cdot)$.
	Moreover, the Lipschitz continuity of $\kappa_i(\cdot, \cdot)$ only requires the kernel to be continuous on the compact domain $\mathbb{X} \times \mathbb{X}$, which is satisfied by most commonly used kernel, such as squared exponential kernel.
	
	\begin{remark}
		The bounded RKHS norm assumption is widely used to obtain certified prediction error bounds in Gaussian process regression \cite{srinivas2012information, hashimoto2022learning}.
		It characterizes the smoothness or complexity of the unknown function with respect to the chosen kernel.
		Since the exact RKHS norm is generally unavailable in practice \cite{tokmak2024pacsbo, tokmak2025safe}, it must be conservatively upper-bounded to preserve the stated learning and control guarantees.
		This may lead to conservative learning and tracking error bounds and may increase the number of triggering events, thereby reducing communication and computational efficiency. 
		Nevertheless, such conservatism does not invalidate the derived tracking error guarantee.
		The proposed control law does not depend on the RKHS norm bound, and the bound is used only to construct the certified prediction error estimate required for the performance analysis and event-trigger design. 
		Finding tighter certified upper bounds on the RKHS norm remains an open problem and is outside the scope of this work. 
		More generally, the proposed computational-delay-aware event-triggered control and online learning framework requires a valid prediction error bound, but not necessarily the specific RKHS-based bound adopted here. 
		Hence, the framework can be combined with other uncertainty quantification methods, provided that they yield valid uniform or closed-loop trajectory-wise prediction error bounds with suitable continuity properties.
	\end{remark}
	
	Given the prior information $m_i(\cdot)$, $\kappa_i(\cdot,\cdot)$ for $i = 1, \cdots, n$ and the data set $\mathbb{D}$ with $N = |\mathbb{D}|$ data pairs, the prediction of the unknown function $\bm{f}(\cdot)$ at the inquiry point $\bm{x} \in \mathbb{X}$ is written as a Gaussian distribution, i.e., $\bm{f}(\bm{x}) \sim \mathcal{N}(\bm{\mu}(\bm{x}, \mathbb{D}), \bm{\Sigma}(\bm{x}, \mathbb{D}))$.
	It is characterized by the posterior mean $\bm{\mu}(\bm{x}, \mathbb{D}) = [\mu_1(\bm{x}, \mathbb{D}), \cdots, \mu_n(\bm{x}, \mathbb{D})]^T$ and variance $\bm{\Sigma}(\bm{x}, \mathbb{D}) = \mathrm{diag} (\sigma_1^2(\bm{x}, \mathbb{D}), \cdots, \sigma_n^2(\bm{x}, \mathbb{D}))$ with
	\begin{align} \label{eqn_GP_prediction}
		&\mu_i(\bm{x}, \mathbb{D}) = \bm{\kappa}_{i,\mathbb{D}}^T(\bm{x}) (\bm{K}_{i,\mathbb{D}} + \sigma_{w,i}^2 \bm{I}_{N})^{-1} \bm{y}_{i,\mathbb{D}}, \\
		&\sigma_i^2(\bm{x}, \mathbb{D}) = \kappa_i(\bm{x}, \bm{x}) - \bm{\kappa}_{i,\mathbb{D}}^T(\bm{x}) (\bm{K}_{i,\mathbb{D}} + \sigma_{w,i}^2 \bm{I}_{N})^{-1} \bm{\kappa}_{i,\mathbb{D}}(\bm{x}), \nonumber
	\end{align}
	where $\bm{\kappa}_{i, \mathbb{D}}(\cdot) = [\kappa_i(\bm{x}^{(1)}, \cdot), \cdots, \kappa_i(\bm{x}^{(N)}, \cdot)]^T$, $\bm{K}_{i,\mathbb{D}} = [\kappa_i(\bm{x}^{(p)}, \bm{x}^{(q)})]_{p,q = 1, \cdots, N}$ and $\bm{y}_{i, \mathbb{D}} = [y_i^{(1)}, \cdots, y_i^{(N)}]^T$ for $i = 1,\cdots, n$.
	The posterior mean $\bm{\mu}(\cdot, \cdot)$ serves as the prediction of $\bm{f}(\cdot)$ and the posterior variance $\bm{\Sigma}(\cdot, \cdot)$ is used to estimate the prediction performance with the following lemma.
	
	\begin{lemma} \label{lemma_GP_error_bound}
		Let \cref{assumption_dataset,assumption_GP} hold for a data set $\mathbb{D}$, that includes a bounded number of training samples, i.e., $| \mathbb{D}(t^{\rm{f}}_p) | = N(t^{\rm{f}}_p) \le \bar{N}$ for $ \forall p \in \mathbb{N}$.
		Choose $\delta \in (0,1)$, then the prediction error for all $p \!\in\! \mathbb{N}$ and $\bm{x}(t^{\rm{f}}_p) \!\in\! \mathbb{X}$ is bounded by
		\begin{align} \label{eqn_GP_error_bound}
			\| \bm{\mu}(\bm{x}(t^{\rm{f}}_p), \mathbb{D}(t^{\rm{f}}_p)) - \bm{f}(\bm{x}(t^{\rm{f}}_p)) \| &\le \eta(\bm{x}(t^{\rm{f}}_p), \mathbb{D}(t^{\rm{f}}_p))
		\end{align}
		with a probability of at least $1 - \delta/2$, where $\eta(\bm{x}, \mathbb{D}) = \sqrt{\mathrm{tr} (\bm{\mathcal{B}} \bm{\Sigma}(\bm{x}, \mathbb{D}))}$, $\bm{\mathcal{B}} \!=\! \mathrm{diag} (\bar{\beta}_{1}^2, \cdots, \bar{\beta}_{n}^2)$ with $\bar{\beta}_{i} \!=\! \Gamma_i + \sigma_{w,i} \sqrt{2 \mathcal{I}_i + 1 + \log(2n / \delta)}$ for $i = 1, \cdots, n$.
		The scalar $\mathcal{I}_i$ denotes the maximal mutual information obtained by $\bar{N}$ noisy samples, i.e., $\mathcal{I}_i = \max_{\bm{z}^{(1)}, \cdots, \bm{z}^{(\bar{N})} \in \mathbb{X}} | \bm{I}_{\bar{N}} + \sigma_{w,i}^{-2} \bm{K}_{i,z}| / 2$ with $\bm{K}_{i,z} = [\kappa_i(\bm{z}^{(p)}, \bm{z}^{(q)})]_{p,q = 1, \cdots, \bar{N}}$.
	\end{lemma}
	\begin{IEEEproof}
		See appendix.
	\end{IEEEproof}
	
	\cref{lemma_GP_error_bound} provides a probabilistic upper error bound for Gaussian process prediction. 
	The constant matrix $\bm{\mathcal{B}}$ requiring solving the maximization problem can be obtained offline, such that the evaluation of error bound $\eta(\cdot,\cdot)$ during operation is only related to the value of posterior variance $\bm{\Sigma}(\cdot,\cdot)$.
	
	Additionally, we consider the prediction performance after an online model update, i.e., after adding a new sample $\{ \bm{x}(t^{\rm{f}}_p), \bm{y}(t^{\rm{f}}_p) \}$ into the training data set at $t^{\rm{f}}_p$.
	Denote $\bm{\mu}^+(\bm{x}(t^{\rm{f}}_p))$ and $\eta^+(\bm{x}(t^{\rm{f}}_p))$ as the posterior mean in \eqref{eqn_GP_prediction} and error bound in \eqref{eqn_GP_error_bound} after online model update with the new data set, which satisfy the following corollary.
	\begin{corollary} \label{corollary_online_GP_error_bound} 
		Let all assumptions in \cref{lemma_GP_error_bound} hold and assume $\{ \bm{x}(t^{\rm{f}}_p), \bm{y}(t^{\rm{f}}_p) \} \in \mathbb{D}(t^{\rm{f}}_p)$ and $N(t^{\rm{f}}_p) \le \bar{N}$.
		Choose $\delta \in (0,1)$, then the prediction error at $\bm{x}(t^{\rm{f}}_p) \in \mathbb{X}$ for all $p \in \mathbb{N}$ is probabilistically bounded by
		\begin{align} \label{eqn_eta_underline}
			\| \bm{f}(\bm{x}(t^{\rm{f}}_p)) \!-\! \bm{\mu}^+\!(\bm{x}(t^{\rm{f}}_p)) \| \!\le\! \eta^+\!(\bm{x}(t^{\rm{f}}_p)) \!\le\! \underline{\eta} \!\!=\!\! (\mathrm{tr} (\bm{\mathcal{B}} \bm{\Sigma}_w))^{\!\frac{1}{2}}
		\end{align}
		with probability of at least $1 \!-\! \delta/2$, where $\eta^+(\bm{x}(t^{\rm{f}}_p)) = (\mathrm{tr} (\bm{\mathcal{B}} (\bm{\Sigma}(\bm{x}(t^{\rm{f}}_p)) + \bm{\Sigma}_w^{-1}) ) )^{1/2}$ with $\bm{\Sigma}_w$ in \cref{assumption_dataset}\footnotemark.
	\end{corollary}
	\begin{IEEEproof}
		See appendix.
	\end{IEEEproof}
	
	\cref{corollary_online_GP_error_bound} bounds the prediction error at $\bm{x}(t^{\rm{f}}_p)$ after online GP model update, which is independent of the complete data set $\mathbb{D}(t^{\rm{f}}_p)$ since $\{ \bm{x}(t^{\rm{f}}_p), \bm{y}(t^{\rm{f}}_p) \} \in \mathbb{D}(t^{\rm{f}}_p)$.
	The positive constant $\underline{\eta}$ acts as the upper bound of the prediction error after online model update, which can be used as the desired inference performance in event-triggered online learning \cite{umlauft2019feedback}.
	
	\footnotetext{
	When $\bm{x}(\cdot)$ and $\mathbb{D}(\cdot)$ have identical timestamp when determining $\bm{\mu}(\cdot,\cdot)$, $\bm{\Sigma}(\cdot,\cdot)$ and $\eta(\cdot,\cdot)$, the notation is simplified by dropping the dependency on $\mathbb{D}$, i.e.,  $\bm{\mu}(\bm{x}(t^{\rm{f}}_p)) \!=\! \bm{\mu}(\bm{x}(t^{\rm{f}}_p), \mathbb{D}(t^{\rm{f}}_p))$.
	}
	
	Note that the bounded number of training samples by $\bar{N}$ required in \cref{lemma_GP_error_bound,corollary_online_GP_error_bound} can be realized by data forgetting strategies.
	A straightforward approach is to delete the oldest data point in the data set, as suggested in \cite{nguyen2008local}.
	Recently, some advanced data deletion methods are developed, such as removing the data point that is farthest from the current query input \cite{yang2025streaming}, or deleting the data that contributes the least to the overall information entropy \cite{umlauft2020smart}.
	It is noted that the results in \cref{corollary_online_GP_error_bound} hold if the forgetting is executed before adding a new data pair.
	Moreover, the bounded size of the data set induces bounded computation time for GP predictions and online model updates \cite{williams2006gaussian}, making \cref{assumption_computational_time} reasonable.
	
	\begin{remark}
		There exist several advanced Gaussian process variants to reduce computational complexity, such as local GP \cite{nguyen2008local}, distributed GP \cite{yang2025streaming} and spase GP \cite{ranganathan2010online}.
		Note that local GP \cite{nguyen2008local} and distributed GP \cite{yang2025streaming} utilize subsets of the data, such that the prediction error bounds in \cref{lemma_GP_error_bound} and \cref{corollary_online_GP_error_bound} are preserved.
		Consequently, the proposed event-triggered control and learning framework can be directly applied to them.
		In contrast, sparse GP \cite{ranganathan2010online} relies on approximate posterior representations with inducing points.
		While computationally efficient, theoretical prediction error bounds for sparse GP remain open, making the direct integration more difficult.
	\end{remark}


	\subsection{Performance of Networked Learning-based Control}
	\label{subsection_control_performance}
	
	In this subsection, the performance of in-networked learning-based control under computational delay is analyzed by observing the closed-loop dynamics, which denotes
	\begin{align} \label{eqn_closed_loop_system}
		\dot{\bm{x}}(t) = \bm{h}(\bm{x}(t), \bm{\pi}(t^{\rm{u}}_k, \bm{x}(t^{\rm{u}}_k), \bm{\mu}(\bm{x}(t^{\rm{f}}_p)))) + \bm{f}(\bm{x}(t))
	\end{align}
	for $t \!\in\! [t^{\rm{u}}_k,\! t^{\rm{u}}_{k\!+\!1})$ and $t^{\rm{u}}_k \!\!\in\! [t^{\rm{f}}_{p\!+\!1},\! t^{\rm{f}}_{p\!+\!2})$.
	According to \cref{assumption_u_performance}, the upper bound of the difference between $\bm{u}(t)$ and $\bm{\pi}(t, \bm{x}(t), \bm{f}(\bm{x}(t)))$, i.e., $\psi(t)$, is firstly investigated as follows.
	
	\begin{lemma} \label{lemma_delta_u}
		Let \cref{assumption_dataset,assumption_GP,assumption_u,assumption_computational_time} hold and choose $\delta \in (0,1)$.
		Then, the difference $\psi(t)$ in \cref{assumption_u_performance} between $\bm{u}(t)$ in \eqref{eqn_u} and $\bm{\pi}(t, \bm{x}(t), \bm{f}(\bm{x}(t)))$ for all $t \in [t^{\rm{u}}_k, t^{\rm{u}}_{k+1})$ and $t^{\rm{u}}_k \in [t^{\rm{f}}_{p+1}, t^{\rm{f}}_{p+2})$ for given $k, p \in \mathbb{N}$ is bounded as
		\begin{align} \label{eqn_delta_u_bound}
			\psi(t) \le \phi(\eta(\bm{x}(t^{\rm{f}}_p)), t^{\rm{u}}_k - t^{\rm{f}}_p, t - t^{\rm{u}}_k)
		\end{align}
		with a probability of at least $1 - \delta$, where $\phi(a,b,c) \!=\! L_{\pi,f} a \!+\! L_{\pi,f} L_f \sqrt{F (b \!+\! c)} \!+\! (L_{\pi,t} \!+\! L_{\pi,x} F) c$ and
		\begin{align} \label{eqn_F}
			F = \bar{h} + L_h \bar{\pi} + L_h L_{\pi,f} \underline{\eta} + (L_h L_{\pi,f} + 1) \bar{\eta}
		\end{align}
		with $\bar{\eta} = (\mathrm{tr} (\bm{\mathcal{B}} \bm{\Sigma}_f) )^{1/2}$, $\bar{h} = \max_{\bm{x} \in \mathbb{X}} \| \bm{h}(\bm{x}, \bm{0}) \|$, $\bm{\Sigma}_f \!= \mathrm{diag}(\sigma_{f,1}^2, \cdots, \sigma_{f,n}^2)$ and $L_f = \sqrt{2 \sum_{i=1}^n \Gamma_i^2 L_{\kappa_i}}$.
	\end{lemma}
	\begin{IEEEproof}
		See appendix.
	\end{IEEEproof}
	
	\cref{lemma_delta_u} shows the upper bound difference between ideal input $\bm{\pi}(t, \bm{x}(t), \bm{f}(\bm{x}(t)))$ and actual input $\bm{u}(t)$ denoted as $\phi(\eta(\bm{x}(t^{\rm{f}}_p)), t^{\rm{u}}_k - t^{\rm{f}}_p, t - t^{\rm{u}}_k)$, whose value depends on prediction accuracy $\eta(\bm{x}(t^{\rm{f}}_p))$, computation time $t^{\rm{u}}_k - t^{\rm{f}}_p$ and control interval $t - t^{\rm{u}}_k)$.
	Note that the bound $\phi(\cdot, \cdot, \cdot)$ in \cref{lemma_delta_u} is suitable for any control update and online learning strategy.
	Following \cref{corollary_online_GP_error_bound} and \cref{assumption_computational_time}, the GP model update conducted at $t^{\rm{f}}_p$ with $p \in \mathbb{N}$ results in $\eta(\bm{x}(t^{\rm{f}}_p)) \le \underline{\eta}$ and $t^{\rm{f}}_{p+1} - t^{\rm{f}}_p \le \bar{\Delta}^+$ respectively, such that for any $k \in \mathbb{N}$ it has
	\begin{align} \label{eqn_tuk_tfp}
		t^{\rm{u}}_k - t^{\rm{f}}_p < t^{\rm{f}}_{p+2} - t^{\rm{f}}_p = \Delta(t^{\rm{f}}_{p+1}) + \Delta(t^{\rm{f}}_p) \le 2 \bar{\Delta}^+
	\end{align}
	with $t^{\rm{u}}_k \in [t^{\rm{f}}_{p+1}, t^{\rm{f}}_{p+2})$.
	Moreover, considering the minimal trigger interval $\varepsilon$ for control update shown in \cref{subsection_introduction_event_triggered_onlineLearning}, the guaranteed upper bound of input difference $\psi(\cdot)$ in \cref{lemma_delta_u} is considered as $\underline{\phi} = \phi(\underline{\eta}, 2 \bar{\Delta}^+, \varepsilon)$.
	Then, the performance for in-network online learning-based control is shown as follows.
	
	\begin{theorem} \label{theorem_control_performance}
		Consider a closed-loop system \eqref{eqn_closed_loop_system} satisfying \cref{assumption_Lh,assumption_u,assumption_u_performance,assumption_dataset,assumption_computational_time,assumption_GP}.
		Choose $\delta \in (0,1)$ and adopt any control and online learning strategy, such that $\psi(t) \le \underline{\phi}$ holds for any $t \in \mathbb{R}_{0,+}$.
		Let $\bar{V}(t)$ be the solution of
		\begin{align} \label{eqn_e_bar_dynamics}
			\dot{\bar{V}}(t) = - \alpha( \bar{V}(t) ) + \gamma(\underline{\phi})
		\end{align}
		with $\bar{V}(0) = V(0)$, then the tracking error $e(t)$ is upper bounded by $\bar{e}(t) = \alpha_1^{-1}(\bar{V}(t))$ with a probability of at least $1 - \delta$, i.e., $e(t) \le \bar{e}(t)$ for all $t \in \mathbb{R}_{0,+}$.
	\end{theorem}
	\begin{IEEEproof}
		See appendix.
	\end{IEEEproof}
	
	\cref{theorem_control_performance} shows the boundness of the tracking error for in-network online learning-based control system under a wide range of control and online learning strategies satisfying $\psi(t) \le \underline{\phi}$, $\forall t \in \mathbb{R}_{0,+}$.
	Note that the solution $\bar{V}(\cdot)$ in \eqref{eqn_e_bar_dynamics} is a function of $\underline{\phi}$ and $V(0)$.
	Let $S(t | s_0, \phi )$ be the solution at $t \in \mathbb{R}_{0,+}$ of the differential equation
	\begin{align} \label{eqn_dynamics_s}
		\dot{s}(t) = - \alpha( s(t) ) + \gamma(\phi)
	\end{align}
	with $s(0) = s_0$, such that $\bar{V}(\cdot)$ in \cref{theorem_control_performance} is written as
	\begin{align}
		\bar{V}(t) = S(t | V(0, e(0)), \underline{\phi} ),
	\end{align}
	resulting in $\bar{e}(t) \le \alpha_1^{-1} (S(t | \alpha_2(e(0)), \underline{\phi} ))$.
	
	Moreover, \cref{theorem_control_performance} provides information about the guaranteed transient behavior of the tracking error, and indicates the ultimate tracking error bound denoted as $e(\infty) \le \alpha_1^{-1}(\alpha^{-1}(\gamma(\underline{\phi})))$.
	Note that the ultimate tracking error bounds derived in \cite{umlauft2019feedback,jiao2022backstepping,dai2023can} are an analytical expression of $\alpha_1^{-1}(\alpha^{-1}(\gamma(\underline{\phi})))$ given the specific system structure and control laws considered in these works.
	
	Note that the boundness of the tracking error $e(\cdot)$ by $\bar{e}(\cdot)$ requires the satisfaction of $\psi(t) \le \underline{\phi}$ for all $t \in \mathbb{R}_{0,+}$.
	In the following subsection, an exemplar control and online learning strategy is given to fulfill the premises in \cref{theorem_control_performance}.

	\subsection{Exemplar Control and Online Learning via Time-trigger}
	\label{subsection_time_trigger}
	
	Due to the expression of $\phi(\cdot,\cdot,\cdot)$ after \eqref{eqn_delta_u_bound}, the boundness $\psi(\cdot) \le \underline{\phi}$ is ensured if $\phi(\eta(\bm{x}(t^{\rm{f}}_p)), t^{\rm{u}}_k - t^{\rm{f}}_p, t - t^{\rm{u}}_k) \le \underline{\phi}$ for any $t \in [t^{\rm{u}}_k, t^{\rm{u}}_{k+1})$ and $t^{\rm{u}}_k \in [t^{\rm{f}}_{p+1}, t^{\rm{f}}_{p+2})$ with $p, k \in \mathbb{N}$.
	Considering the definition of $\underline{\phi}$ before \cref{theorem_control_performance}, a naive strategy is to design control and online learning separately via time-triggered mechanism.
	Specifically, design
	\begin{itemize}
		\item the time-triggered control with $t^{\rm{u}}_{k+1} - t^{\rm{u}}_k = \varepsilon$, $\forall k \in \mathbb{N}$, such that $t - t^{\rm{u}}_k \le t^{\rm{u}}_{k+1} - t^{\rm{u}}_k = \varepsilon$, $\forall t \in [t^{\rm{u}}_k,t^{\rm{u}}_{k+1})$;
		\item the time-triggered online learning, such that the data pair $\{ \bm{x}(t^{\rm{f}}_p), \bm{y}(t^{\rm{f}}_p)\}$ is added to $\mathbb{D}$ at $t^{\rm{f}}_p$ for each $p \in \mathbb{N}$.
	\end{itemize}
	Consider the online learning performance shown in \cref{corollary_online_GP_error_bound} and the bounded computation time in \cref{assumption_computational_time}, it is obvious to see that the time-triggered online learning ensures $\eta(\bm{x}(t^{\rm{f}}_p)) \le \underline{\eta}$ and $t^{\rm{u}}_k - t^{\rm{f}}_p \le 2 \bar{\Delta}^+$ as in \eqref{eqn_tuk_tfp} for any $p, k \in \mathbb{N}$ and $t^{\rm{u}}_k \in [t^{\rm{f}}_{p+1}, t^{\rm{f}}_{p+2})$.
	Therefore, the upper boundness of $\phi(\eta(\bm{x}(t^{\rm{f}}_p)), t^{\rm{u}}_k - t^{\rm{f}}_p, t - t^{\rm{u}}_k) \le \underline{\phi}$ by $\underline{\phi}$ is guaranteed, inducing $\psi(t) \le \underline{\phi}$ for any $t \in [0, t^{\rm{f}}_1)$.
	
	Additionally, noting that the computation of $\bm{\mu}(\bm{x}(t^{\rm{f}}_0))$ is unfinished in $t \in [0, t^{\rm{f}}_1)$, the compensation of $\bm{f}(\bm{x}(t))$ relies on the initial guess denoted as $\hat{\bm{f}}_{-1} \in \mathbb{R}^n$.
	To ensure the boundness of $\psi(t)$ with $t \in [0, t^{\rm{f}}_1)$ by $\underline{\phi}$, some specific requirement on $\hat{\bm{f}}_{-1}$ is required shown as follows.
	
	\begin{corollary} \label{corollary_time_trigger_performance}
		Let all assumptions in \cref{theorem_control_performance} hold and apply the time-triggered control and online learning strategy.
		Choose $\delta \in (0,1)$ and apply $\bm{u}(t) = \bm{\pi}(t, \bm{x}(t), \hat{\bm{f}}_{-1})$ for all $t \in [t^{\rm{f}}_0, t^{\rm{f}}_1)$ with the initial guess $\hat{\bm{f}}_{-1}$ satisfying
		\begin{align} \label{eqn_initial_prediction_error_bound}
			\Pr \{ \| \bm{f}(\bm{x}(t^{\rm{f}}_0)) - \hat{\bm{f}}_{-1} \| \le \eta_{-1} \} \ge 1 - \delta,
		\end{align}
		where $\eta_{-1} = \underline{\eta} + L_f F^{1/2} (\sqrt{2 \bar{\Delta}^+ + \varepsilon} - \sqrt{\bar{\Delta}^+ + \varepsilon})$ with $\bar{\Delta}^+$ in \eqref{eqn_Delta_bar}.
		Then, the boundness $e(t) \le \bar{e}(t)$ holds for any $t \in \mathbb{R}_{0,+}$ with a probability of at least $1 - \delta$.
	\end{corollary}
	\begin{IEEEproof}
		See appendix.
	\end{IEEEproof}
	
	\cref{corollary_time_trigger_performance} shows the same performance as in \cref{theorem_control_performance} using the time-triggered control and online learning strategy and initial condition \eqref{eqn_initial_prediction_error_bound}.
	Note that the condition \eqref{eqn_initial_prediction_error_bound} can be realized by adding the data pair $\{ \bm{x}(0), \bm{y}(0) \}$ into the data set $\mathbb{D}(0)$ at $t = 0$ and generating $\bm{f}_{-1}$ as $\bm{f}_{-1} = \bm{\mu}(\bm{x}(0), \mathbb{D}(0))$.
	Therefore, the time-triggered control and online learning strategy with initial condition is practically implementable.
	
	While time-triggered control and online learning mechanism is proven to achieve the desired tracking performance with $\bar{e}(\cdot)$ in \cref{corollary_time_trigger_performance} by satisfying the condition $\psi(\cdot) \le \underline{\phi}$ in \cref{theorem_control_performance}, it ignores learning and control efficiency.
	In the next section, a novel design of control and online learning strategy for both high communication and computation efficiencies are shown under the computational delay.

	\section{Asynchronous Event-triggered Control and \\ Online Learning under Computational Delay}
	\label{section_double_trigger}
	
	In this section, a more efficient online learning and control strategy with event-trigger mechanism is considered\footnote{The notations used in this section is summarized in \cref{table_notation}.}.
	First, the effect of computational delay on the event-trigger design is discussed in \cref{subsection_effect_delay_on_trigger}, which serves as a guidance for asynchronous event-triggered control and online learning design in \cref{subsection_double_trigger_design} with performance guarantee.
	Then, \cref{subsection_double_trigger_exp_stable} shows an example with exponential stability, such that the proposed event-trigger exhibits a closed form.

	\begin{table}[t] 
		\centering
		\caption{Summary of Main Notations in \cref{section_double_trigger}}
		\label{table_notation}
		\begin{tabular}{ll}
			\hline
			Symbol & Description \\ 
			\hline
			$\Delta^{\rm{p}}(\cdot), \Delta^{\rm{u}}(\cdot)$ & Computation time for prediction and update \\
			$\bar{\Delta}^{\rm{p}}, \bar{\Delta}^{\rm{u}}$ & Upper bounds of $\Delta^{\rm{p}}(\cdot), \Delta^{\rm{u}}(\cdot)$ \\
			$\Delta(\cdot), \bar{\Delta}$ & Total computation time and its upper bound \\
			$\{ t^{\rm{f}}_p \}_{p \in \mathbb{N}}, \{ t^{\rm{u}}_k \}_{k \in \mathbb{N}}$ & Time sequence for online learning and control \\
			$V(\cdot), \bar{V}(\cdot)$ & Lyapunov function and its upper bound \\
			$\phi_{p,1}\!(\cdot), \phi_{p,2}\!(\cdot, \cdot)$ & Value of $\phi(\cdot, \cdot, \cdot)$ in $[t^{\rm{f}}_p,\! t^{\rm{f}}_{p+1})$ and $[t^{\rm{f}}_{p+1},\! t^{\rm{f}}_{p+2})$ \\
			$\hat{V}_{p,1}\!(\cdot, \cdot), \hat{V}_{p,2}\!(\cdot, \cdot, \cdot)$ & Bound of $V(\cdot)$ in $[t^{\rm{f}}_p,\! t^{\rm{f}}_{p+1})$ and $[t^{\rm{f}}_{p+1},\! t^{\rm{f}}_{p+2})$ \\
			\hline
		\end{tabular}
	\end{table}

	\subsection{Analysis of Computational Delay on the Event-trigger}
	\label{subsection_effect_delay_on_trigger}
	
	In this subsection, we aim to answer the question ``whether to update the controller or machine learning model when performance is not sufficiently good'' in \cref{section_introduction} by analyzing the effects of online learning and control updates on the instant evolution trend of the value of Lyapunov function.
	Specifically, combining the result in \cref{lemma_delta_u} and \cref{theorem_control_performance}, the time derivative of the Lyapunov function $V(\cdot, \cdot)$ is bounded by
	\begin{align}
		\dot{V}(t) \le - \alpha( V(t) ) + \gamma\big(\phi\big(\eta(\bm{x}(t^{\rm{f}}_p)), t^{\rm{f}}_{p+2} - t^{\rm{f}}_p, t - t^{\rm{u}}_k\big)\big) 
	\end{align}
	with $k, p \in \mathbb{N}$ satisfying $t \in [t^{\rm{u}}_k, t^{\rm{u}}_{k+1})$ and $t^{\rm{u}}_k \in [t^{\rm{f}}_{p+1}, t^{\rm{f}}_{p+2})$.
	Intuitively, a good control input $\bm{u}(t)$ indicates small $\phi(\eta(\bm{x}(t^{\rm{f}}_p))$, $t^{\rm{f}}_{p+2} - t^{\rm{f}}_p$ and $t - t^{\rm{u}}_k$ to maintain the boundness of tracking error by $\bar{V}(t)$ in \cref{theorem_control_performance}.
	Note that it does not force $\phi(\eta(\bm{x}(t^{\rm{f}}_p)), t^{\rm{f}}_{p+2} - t^{\rm{f}}_p, t - t^{\rm{u}}_k) \le \underline{\phi}$ for any $t \in \mathbb{R}_{0,+}$ as in the discussion of \cref{lemma_delta_u}.
	Instead, it only requires $\phi(\eta(\bm{x}(t^{\rm{f}}_p)), t^{\rm{f}}_{p+2} - t^{\rm{f}}_p, t - t^{\rm{u}}_k) \le \underline{\phi}$ if $V(t) = \bar{V}(t)$.
	Considering $t^{\rm{f}}_{p+2} - t^{\rm{f}}_p = \Delta(t^{\rm{f}}_p) + \Delta(t^{\rm{f}}_{p+1})$, the value of $\phi(\eta(\bm{x}(t^{\rm{f}}_p)), t^{\rm{f}}_{p+2} - t^{\rm{f}}_p, t - t^{\rm{u}}_k)$ is composed of $3$ parts:
	\begin{itemize}
		\item $\eta(\bm{x}(t^{\rm{f}}_p))$ and $\Delta(t^{\rm{f}}_p)$, which depend on the online learning decision at previous time $t^{\rm{f}}_p$ and are unchangeable at $t$;
		\item $t - t^{\rm{u}}_k$, which is induced by intermittent control and becomes $0$ if $\bm{u}(t)$ is updated at $t$ inducing $t^{\rm{u}}_{k+1} = t$;
		\item $\Delta(t^{\rm{f}}_{p+1})$, which requires the information of future time instance $t^{\rm{f}}_{p+2}$ and thus is considered as unknown at $t$.
	\end{itemize}
	While the information of $t^{\rm{f}}_{p+2}$ is inaccessible at $t$, the conservative upper bound as $\Delta(t^{\rm{f}}_{p+1}) \le \bar{\Delta}^+$ is available according to \cref{assumption_computational_time}, such that $t^{\rm{f}}_{p+2} \le t^{\rm{f}}_{p+1} + \bar{\Delta}^+$ and $\phi(\eta(\bm{x}(t^{\rm{f}}_p)), t^{\rm{f}}_{p+2} - t^{\rm{f}}_p, t - t^{\rm{u}}_k) \le \phi(\eta(\bm{x}(t^{\rm{f}}_p)), \Delta(t^{\rm{f}}_p) + \bar{\Delta}^+, t - t^{\rm{u}}_k)$.
	As discussed before, the value of $\phi(\eta(\bm{x}(t^{\rm{f}}_p)), \Delta(t^{\rm{f}}_p) + \bar{\Delta}^+, t - t^{\rm{u}}_k)$ relies on the behavior of online learning and control at different time instances, i.e., $t^{\rm{f}}_p$ and $t$, such that the question ``whether to update the controller or machine learning model when performance is not sufficiently good'' in \cref{section_introduction} is answered from two perspectives shown as follows:
	\begin{itemize}
		\item when the control performance is below the expectation at $t$, i.e., $V(t) > \bar{V}(t)$, the control input $\bm{u}(t)$ is updated such that $t - t^{\rm{u}}_k = 0$;
		\item a GP model update at $t^{\rm{f}}_p$ ensures that the control update at $t$ achieves sufficiently good performance.
	\end{itemize}
	The requirement of ``sufficiently good performance'' for the GP prediction is explained in detail later.
	Next, the above answer is formulated mathematically.
	
	\subsubsection{Predictive Model Update}
	As shown in \cref{figure_computational_delay}, the online learning decision at $t^{\rm{f}}_p$ should guarantee ``sufficiently good performance'' in period $[t^{\rm{f}}_p, t^{\rm{f}}_{p+2})$, therefore the future evolution of the Lyapunov function $V(\cdot)$ is investigated.
	Considering the effects of GP online learning shown in \cref{figure_computational_delay}, we estimate the value of the Lyapunov function in $[t^{\rm{f}}_p, t^{\rm{f}}_{p+1})$ and $[t^{\rm{f}}_{p+1}, t^{\rm{f}}_{p+2})$ separately at $t^{\rm{f}}_p$ for all $p \in \mathbb{N}$.
	
	For the estimation in $[t^{\rm{f}}_p, t^{\rm{f}}_{p+1})$, the previous information $\eta(\bm{x}(t^{\rm{f}}_{p-1}))$ and $\Delta(t^{\rm{f}}_{p-1}) = t^{\rm{f}}_p - t^{\rm{f}}_{p-1}$ is available.
	Additionally, the estimation of $V(t)$ requires predefined control behavior.
	Taking the performance of the time-triggered control in \cref{subsection_time_trigger} as a baseline, the estimated value of the Lyapunov function $V(t)$ for $t \in [t^{\rm{f}}_p, t^{\rm{f}}_{p+1})$ is upper bounded by
	\begin{align} \label{eqn_hat_V_1_1}
		V(t) \le& S(t - t^{\rm{f}}_p | V(t^{\rm{f}}_p), \phi(\eta(\bm{x}(t^{\rm{f}}_{p-1})), \Delta(t^{\rm{f}}_{p-1}) + \Delta(t^{\rm{f}}_p), \varepsilon) ) \nonumber \\
		=& S(t - t^{\rm{f}}_p | V(t^{\rm{f}}_p), \phi_{p,1}(\Delta(t^{\rm{f}}_p)) ), 
	\end{align}
	where $\phi_{p,1}(\Delta(t^{\rm{f}}_p)) = \phi(\eta(\bm{x}(t^{\rm{f}}_{p-1})), \Delta(t^{\rm{f}}_{p-1}) + \Delta(t^{\rm{f}}_p), \varepsilon)$ is related to the computational time $\Delta(t^{\rm{f}}_p)$ starting at $t^{\rm{f}}_p$.
	Define $\hat{V}_{p,1}(t - t^{\rm{f}}_p, \Delta(t^{\rm{f}}_p) ) = S(t - t^{\rm{f}}_p | V(t^{\rm{f}}_p), \phi_{p,1}(\Delta(t^{\rm{f}}_p)) )$, the inequality in \eqref{eqn_hat_V_1_1} becomes
	\begin{align} \label{eqn_hat_V_1}
		V(t) \le \hat{V}_{p,1}(t - t^{\rm{f}}_p, \Delta(t^{\rm{f}}_p) )
	\end{align}
	for any $t \in [t^{\rm{f}}_p, t^{\rm{f}}_{p+1})$.
	Similarly, the Lyapunov function $V(t)$ with $t \!\in\! [t^{\rm{f}}_{p+1}, t^{\rm{f}}_{p+2})$ is written as
	\begin{align} \label{eqn_hat_V_2_1}
		V\!(t) \!\!\le& S\!(t \!\!-\!\! t^{\rm{f}}_{p\!+\!1} | V\!(t^{\rm{f}}_{p\!+\!1}),\! \phi(\eta(\bm{x}(t^{\rm{f}}_p)), \Delta(t^{\rm{f}}_p) \!\!+\!\! \Delta(t^{\rm{f}}_{p\!+\!1}),\! \varepsilon) ) \\
		\le& S\!(t \!\!-\!\! t^{\rm{f}}_p \!\!-\!\! \Delta(t^{\rm{f}}_p) \!|\! \hat{V}_{p,\! 1}\!(\!\Delta(t^{\rm{f}}_p),\! \Delta(t^{\rm{f}}_p)),\! \phi_{p,\! 2}\!(\eta(\bm{x}(t^{\rm{f}}_p))\!,\! \Delta\!(t^{\rm{f}}_p) \!) \!), \nonumber
	\end{align}
	where $\phi_{p,2}(\eta(\bm{x}(t^{\rm{f}}_p)), \Delta(t^{\rm{f}}_p)) = \phi(\eta(\bm{x}(t^{\rm{f}}_p)), \Delta(t^{\rm{f}}_p) + \bar{\Delta}^+, \varepsilon)$.
	For simplicity, let $\hat{V}_{p,2}(t - t^{\rm{f}}_p - \Delta(t^{\rm{f}}_p), \Delta(t^{\rm{f}}_p), \eta(\bm{x}(t^{\rm{f}}_p)) ) = S(t - t^{\rm{f}}_p - \Delta(t^{\rm{f}}_p) | \hat{V}_{p,1}(\Delta(t^{\rm{f}}_p), \Delta(t^{\rm{f}}_p)), \phi_{p,2}(\eta(\bm{x}(t^{\rm{f}}_p)), \Delta(t^{\rm{f}}_p)) )$, then the inequality in \eqref{eqn_hat_V_2_1} is further written as
	\begin{align} 
		V(t) \le \hat{V}_{p,2}(t - t^{\rm{f}}_p - \Delta(t^{\rm{f}}_p), \Delta(t^{\rm{f}}_p), \eta(\bm{x}(t^{\rm{f}}_p)) ).
	\end{align}
	Unlike the estimation of $V(t)$ for $t \in [t^{\rm{f}}_p, t^{\rm{f}}_{p+1})$ in \eqref{eqn_hat_V_1}, the upper bound $\hat{V}_{p,2}(t - t^{\rm{f}}_p - \Delta(t^{\rm{f}}_p), \Delta(t^{\rm{f}}_p), \eta(\bm{x}(t^{\rm{f}}_p)) )$ is not only related to the time $t - t^{\rm{f}}_p - \Delta(t^{\rm{f}}_p)$ and computational delay $\Delta(t^{\rm{f}}_p)$ but also includes the prediction error bound $\eta(\bm{x}(t^{\rm{f}}_p))$, which depends on the online learning decision at $t^{\rm{f}}_p$.
	
	For safety critical control with guaranteed tracking error bound $\bar{e}(\cdot)$ in \cref{theorem_control_performance}, the behavior of the GP model from the event-triggered mechanism should ensure upper bounds of the Lyapunov function obtained at $t^{\rm{f}}_p$, i.e., $\hat{V}_{p,1}(t - t^{\rm{f}}_p, \Delta(t^{\rm{f}}_p))$ and $\hat{V}_{p,2}(t - t^{\rm{f}}_p - \Delta(t^{\rm{f}}_p), \Delta(t^{\rm{f}}_p), \eta(\bm{x}(t^{\rm{f}}_p)) )$, are bounded by $\bar{V}(t)$.
	The above description for ``sufficiently good performance'' for GP model is formally formulated as follows.
	\begin{property} \label{property_ET_learning}
		Given the desired $\bar{V}(\cdot)$ in \cref{theorem_control_performance}, the event-triggered online learning mechanism ensures that
		\begin{align}
			&\hat{V}_{p,1}(t \!-\! t^{\rm{f}}_p, \Delta(t^{\rm{f}}_p) ) \!\le\! \bar{V}(t)
		\end{align}
		holds for all $t \in [t^{\rm{f}}_p, t^{\rm{f}}_p + \Delta(t^{\rm{f}}_p))$, $\forall p \in \mathbb{N}$ and
		\begin{align}
			&\hat{V}_{p,2}(t \!-\! t^{\rm{f}}_p \!-\! \Delta(t^{\rm{f}}_p), \Delta(t^{\rm{f}}_p), \eta(\bm{x}(t^{\rm{f}}_p)) ) \!\le\! \bar{V}(t) 
		\end{align}
		for all $t \in [t^{\rm{f}}_p + \Delta(t^{\rm{f}}_p), t^{\rm{f}}_p + \Delta(t^{\rm{f}}_p) + \Delta(t^{\rm{f}}_{p+1}) )$ with	any possible $\Delta(t^{\rm{f}}_p)$ and $\Delta(t^{\rm{f}}_{p+1})$ satisfying \cref{assumption_computational_time}.
	\end{property}

	Note that the estimation of the Lyapunov function value is based on the time-triggered mechanism of controller in \cref{subsection_time_trigger}.
	To alleviate the communication burden by transmitting less commands over network, the event-triggered control strategy is aimed for and discussed later.
	
	\subsubsection{Instant Control Update}
	
	As discussed previously, the control update at $t$ induces $t^{\rm{u}}_k = t$ immediately with $k \in \mathbb{N}$, saving the estimation phase for future $V(\cdot)$ as GP model update.
	To ensure the validity of the upper bound $\hat{V}_{p,1}(\cdot,\cdot)$ and $\hat{V}_{p,2}(\cdot,\cdot,\cdot)$ for the estimated Lyapunov function, the control update strategy is required to ensure $V(t) \le \hat{V}_{p,1}(t - t^{\rm{f}}_p, \Delta(t^{\rm{f}}_p) )$ and $V(t) \le \hat{V}_{p-1,2}(t - t^{\rm{f}}_{p-1}, \Delta(t^{\rm{f}}_{p-1}), \eta(\bm{x}(t^{\rm{f}}_{p-1})) )$ for any $t \in [t^{\rm{f}}_p, t^{\rm{f}}_{p+1})$ and $p \in \mathbb{N}_+$.
	Considering more conservatism in $\hat{V}_{p-1,2}(\cdot,\cdot,\cdot)$ than in $\hat{V}_{p,1}(\cdot,\cdot)$ due to the lack of information for $\Delta(t^{\rm{f}}_{p-1})$, only the first requirement for $V(\cdot)$ is considered, which is formally formulated as follows.

	\begin{property} \label{property_ET_control}
		Given the expression of $\hat{V}_{p,1}(\cdot, \cdot)$ at $t^{\rm{f}}_p$ with $p \in \mathbb{N}$, the event-triggered control mechanism guarantees $V(t) \!\le\! \hat{V}_{p,1}(t - t^{\rm{f}}_p, \Delta(t^{\rm{f}}_p) )$ holds for all $t \in [t^{\rm{f}}_p, t^{\rm{f}}_{p+1})$.
	\end{property}

	Inspired by the above discussion, the event-triggered mechanism for asynchronous control and online learning is designed in \cref{subsection_double_trigger_design} with performance analysis.
	
	\subsection{Event-trigger Design and Performance Analysis}
	\label{subsection_double_trigger_design}
	
	Based on the requirements in \cref{property_ET_control,property_ET_learning}, in this subsection a general form of event-triggers for control and online learning is proposed.
	
	Considering the model update decision requires the future evolution estimation for the Lyapunov function, which is based on the control update strategy, we first devise the event-triggered control mechanism from \cref{property_ET_control}.
	Specifically, the time sequence $\{ t^{\rm{u}}_k \}_{k \in \mathbb{N}}$ is determined as $t^{\rm{u}}_0 = 0$ and
	\begin{align} \label{eqn_control_update_condition}
		t^{\rm{u}}_{k+1} = \inf\{ t > t^{\rm{u}}_k: \rho^{\rm{u}}(t) \ge 0 \},
	\end{align}
	where $\rho^{\rm{u}}(t): \mathbb{R}_{0,+} \to \mathbb{R}$ is the trigger function designed as
	\begin{align} \label{eqn_rho_u}
		\rho^{\rm{u}}(t) = V(t) - \hat{V}_{p,1}(t - t^{\rm{f}}_p, t - t^{\rm{f}}_p)
	\end{align}
	for all $t \in [t^{\rm{f}}_p, t^{\rm{f}}_{p+1})$ with $p \in \mathbb{N}$ and .
	The devised event-triggered control strategy induces the following performance.
	
	\begin{lemma} \label{lemma_ET_control}
		Let all assumptions in \cref{theorem_control_performance} hold with event-trigger \eqref{eqn_rho_u} for control updates and any online learning strategy.
		Then, it holds for any $p \in \mathbb{N}$
		\begin{enumerate}
			\item $V(t) \le \hat{V}_{p,1}(t - t^{\rm{f}}_p, \Delta(t^{\rm{f}}_p) )$ for all $t \in [t^{\rm{f}}_p, t^{\rm{f}}_{p+1})$;
			\item $\hat{V}_{p,2}(t \!-\! t^{\rm{f}}_{p + 1},\! \Delta(t^{\rm{f}}_p),\! \eta(\bm{x}(t^{\rm{f}}_p)) ) \!\ge\! \hat{V}_{p + 1,1}(t \!-\! t^{\rm{f}}_{p + 1},\! \Delta(t^{\rm{f}}_{p + 1}) )$ 
			
			for all $t \in [t^{\rm{f}}_{p+1}, t^{\rm{f}}_{p+2})$ 
		\end{enumerate}
		with a probability of at least $1 - \delta$.
	\end{lemma}
	\begin{IEEEproof}
		See appendix.
	\end{IEEEproof}
	
	\cref{lemma_ET_control} shows the boundness of the actual Lyapunov function $V(t)$ by its estimation $\hat{V}_{p,1}(t - t^{\rm{f}}_p, \Delta(t^{\rm{f}}_p))$.
	Moreover, it also proves that $\hat{V}_{p,2}(t - t^{\rm{f}}_{p+1}, \Delta(t^{\rm{f}}_p), \eta(\bm{x}(t^{\rm{f}}_p)) )$ is a conservative estimate of $\hat{V}_{p + 1,1}(t - t^{\rm{f}}_{p + 1}, \Delta(t^{\rm{f}}_{p + 1}) )$ considering the second statement in \cref{lemma_ET_control}.
	With \cref{lemma_ET_control}, the requirement for online learning in \cref{property_ET_learning} can be simplified as
	\begin{align} \label{eqn_V_hat_V}
		\hat{V}(t) \le \bar{V}(t), &&\forall t \in \mathbb{R}_{0,+},
	\end{align}
	where the function $\hat{V}(\cdot): \mathbb{R}_{0,+} \to \mathbb{R}$ is defined as $\hat{V}(t) = \hat{V}_{0,1}(t, \Delta(t^{\rm{f}}_0))$ for $t \in [0, t^{\rm{f}}_1)$ with $\hat{V}_{0,1}(\cdot,\cdot)$ in \eqref{eqn_hat_V_1} and
	\begin{align} \label{eqn_hat_V}
		\hat{V}(t) = \hat{V}_{p,2}(t - t^{\rm{f}}_{p + 1} , \Delta(t^{\rm{f}}_p), \eta(\bm{x}(t^{\rm{f}}_p)) ),
	\end{align}
	for $t \in [t^{\rm{f}}_{p+1}, t^{\rm{f}}_{p+2})$ with $p \in \mathbb{N}$.
	Due to the piece-wise form of $\hat{V}(t)$, the inequality in \eqref{eqn_V_hat_V} is based on the initial condition on $\bm{f}_{-1}$ discussed in \cref{corollary_time_trigger_performance} for $t \in [0, t^{\rm{f}}_1)$, and the design of event-triggered online learning for $t \in [t^{\rm{f}}_1, \infty)$.
	Specifically, the GP model represented by the employed data set is updated following event-triggered online learning strategy as
	\begin{align} \label{eqn_dataset_ET_update}
		\mathbb{D}(t^{\rm{f}}_p) \!=\! \begin{cases}
			\{ \mathbb{D}(t^{\rm{f}}_{p-1}), \{ \bm{x}(t^{\rm{f}}_p), \bm{y}(t^{\rm{f}}_p) \} \} &\!\!\!\!, \text{if} ~ \rho^{\rm{f}}(t^{\rm{f}}_p) \!\ge\! 0 \\
			\mathbb{D}(t^{\rm{f}}_{p-1}) &\!\!\!\!, \text{otherwise}
		\end{cases}
	\end{align}
	for $p \in \mathbb{N}$ with $\mathbb{D}(t^{\rm{f}}_{-1}) = \mathbb{D}_{-1}$ as the initial data set, where $\rho^{\rm{f}}(\cdot): \mathbb{R}_{0,+} \to \mathbb{R}$ denotes the trigger function.
	Combining the event-triggered online learning strategy in \eqref{eqn_dataset_ET_update} with \cref{assumption_computational_time}, the upper bound of the computation time $\Delta(t^{\rm{f}}_p)$ for all $p \in \mathbb{N}$ is written as
	\begin{align} \label{eqn_Delta_bar}
		\Delta(t^{\rm{f}}_p) \le \bar{\Delta}(t^{\rm{f}}_p) = \begin{cases}
			\bar{\Delta}^+ = \bar{\Delta}^{\rm{p}} + \bar{\Delta}^{\rm{u}} &, \text{if}~\rho^{\rm{f}}(t^{\rm{f}}_p) > 0 \\
			\bar{\Delta}^{\rm{p}} &, \text{otherwise}
		\end{cases}.
	\end{align}
	Considering the different upper bound of computation time in \eqref{eqn_Delta_bar} for cases with/without online model update, the event-trigger function $\rho^{\rm{f}}(\cdot)$ in \eqref{eqn_dataset_ET_update} is designed for all $p \in \mathbb{N}$ as
	\begin{align} \label{eqn_rho_f}
		\rho^{\rm{f}}(t^{\rm{f}}_p) = &\max\nolimits_{ \tilde{\Delta}_1 \in [0, \bar{\Delta}^{\rm{p}}], \tilde{\Delta}_2 \in [0, \bar{\Delta}^+] }  \\
		&~~~~~~ \big( \hat{V}_{p,2}^-(\tilde{\Delta}_1, \tilde{\Delta}_2) - \bar{V}(t^{\rm{f}}_p + \tilde{\Delta}_1 + \tilde{\Delta}_2) \big), \nonumber
	\end{align}
	where $\hat{V}_{p,2}^-\!(\tilde{\Delta}_1,\! \tilde{\Delta}_2) \!=\! \hat{V}_{p,2}\!(\tilde{\Delta}_2, \tilde{\Delta}_1, \eta^-\!(\bm{x}(t^{\rm{f}}_p)))$ and $\eta^-(\bm{x}(t^{\rm{f}}_p))$ is the prediction error bound at $\bm{x}(t^{\rm{f}}_p)$ using $\mathbb{D}(\bm{x}(t^{\rm{f}}_{p-1}))$.
	Intuitively, $\hat{V}_{p,2}^-(\tilde{\Delta}_1, \tilde{\Delta}_2)$ represents the estimated value of Lyapunov function $V(\cdot)$ at $t^{\rm{f}}_p + \tilde{\Delta}_1 + \tilde{\Delta}_2$ without online learning at $t^{\rm{f}}_p$.
	Moreover, considering $\bar{V}(\cdot)$ is the desired upper bound of $V(\cdot)$, a positive $\rho^{\rm{f}}(t^{\rm{f}}_p)$ indicates the desired tracking performance is not guaranteed.	
	
	Moreover, the following performance holds with the proposed event-triggered online learning with \eqref{eqn_rho_f}.
	
	\begin{lemma} \label{lemma_ET_learning}
		Let all assumptions in \cref{theorem_control_performance} hold.
		If $\hat{V}_{p,1}(\Delta(t^{\rm{f}}_p), \Delta(t^{\rm{f}}_p)) \le \bar{V}(t^{\rm{f}}_p + \Delta(t^{\rm{f}}_p))$ with $p \in \mathbb{N}$, the event-triggered online learning mechanism ensures $\hat{V}(t) \le \bar{V}(t)$ for all $t \in [t^{\rm{f}}_p + \Delta(t^{\rm{f}}_p), t^{\rm{f}}_p + \Delta(t^{\rm{f}}_p) + \bar{\Delta}^+)$.
	\end{lemma}
	\begin{IEEEproof}
		See appendix.
	\end{IEEEproof}
	
	Next, the joint performance of the closed loop system using the learning-based control law $\bm{\pi}(\cdot,\cdot,\cdot)$ and event-triggered mechanism with \eqref{eqn_rho_u} and \eqref{eqn_rho_f} is shown as follows.
	
	\begin{theorem} \label{theorem_double_trigger}
		Consider a dynamical system \eqref{eqn_general_system} under \cref{assumption_Lh}, which is controlled to track a pre-defined trajectory with the control law $\bm{\pi}(\cdot, \cdot, \cdot)$ satisfying \cref{assumption_u,assumption_u_performance}.
		The unknown function $\bm{f}(\cdot)$ is predicted via Gaussian process satisfying \cref{assumption_GP} using a data set satisfying \cref{assumption_dataset}, whose computation time follows \cref{assumption_computational_time}.
		Suppose the event-triggered mechanism for control and online learning is designed with \eqref{eqn_rho_u} and \eqref{eqn_rho_f}.
		Employ $\hat{\bm{f}}_{-1} \in \mathbb{R}^n$ satisfying \eqref{eqn_initial_prediction_error_bound} in $\bm{\pi}(\cdot,\cdot,\cdot)$ during $[t^{\rm{f}}_0, t^{\rm{f}}_1)$, then it holds
		\begin{align}
			V(t) \le \hat{V}(t) \le \bar{V}(t)
		\end{align}
		with $\hat{V}(\cdot)$, $\bar{V}(\cdot)$ in \eqref{eqn_hat_V}, \eqref{eqn_e_bar_dynamics} respectively, and therefore $e(t) \le \bar{e}(t)$ for all $t \in \mathbb{R}_{0,+}$ with a probability of at least $1 - \delta$.
		Furthermore, the trigger intervals for both control and GP update are lower bounded by a strictly positive constant with probability of at least $1 - \delta$.
	\end{theorem}
	\begin{IEEEproof}
		See appendix.
	\end{IEEEproof}
	
	\cref{theorem_double_trigger} shows that the control performance in terms of the bounded tracking error in \cref{theorem_control_performance} is achieved through the proposed event-triggered strategy.
	Note that the strict positivity of $\varepsilon$ ensures strictly positive inter-event times between two transmission of control inputs in $[t^{\rm{f}}_p, t^{\rm{f}}_{p+1})$ for any $p \in \mathbb{N}$, avoiding infinite transmission in any finite time interval resulting in the exclusion of Zeno behavior.
	
	\begin{remark}
		In this work, the event-triggered control law \eqref{eqn_control_update_condition} is applied to the sensor-to-controller (S-C) channel as in \cref{figure_networked_architecture}, to reduce the frequency of state transmissions.
		To alleviate the communication burden at the controller-to-actuator (C-A) channel, its transmission strategy can be designed such that control inputs are transmitted only when a new state is received or GP prediction is updated.
		Note that the proposed event-triggered control mechanism can also be deployed at the C-A channel. 
		However, this may lead to a large computational burden at the S-C channel, due to the continuous state information required to evaluate the triggering condition \eqref{eqn_control_update_condition}.
	\end{remark}
	
	Note that the evaluation of the trigger functions $\rho^{\rm{u}}(\cdot)$ in \eqref{eqn_rho_u} and $\rho^{\rm{f}}(\cdot)$ in \eqref{eqn_rho_f} requires a numerical evaluation of $S(\cdot|\cdot,\cdot)$ and solving the optimization problem in \eqref{eqn_rho_f}, which induces additional computation time.
	In the next subsection, the special case of exponentially stable control laws is considered to derive closed forms for $\rho^{\rm{u}}(\cdot)$ and $\rho^{\rm{f}}(\cdot)$.
	
	\subsection{Event-Trigger for Exponentially Stabilizable System}
	\label{subsection_double_trigger_exp_stable}
	
	In this subsection, the closed form of the tracking error bound $\bar{e}(\cdot)$ and the trigger functions $\rho^{\rm{f}}(\cdot)$, $\rho^{\rm{u}}(\cdot)$ are obtained by considering specific functions for $\alpha_1(\cdot)$, $\alpha_2(\cdot)$, $\alpha(\cdot)$ and $\gamma(\cdot)$ in \cref{assumption_u_performance} as follows.
	\begin{assumption} \label{assumption_alpha_exp_stable}
		There exist constants $\lambda_1, \lambda_2, \lambda, \lambda_{\gamma} \in \mathbb{R}_+$ such that $\alpha_1(s) = \lambda_1 s^2$, $\alpha_2(s) = \lambda_2 s^2$, $\alpha(s) = \lambda s$ and $\gamma(s) = \lambda_{\gamma} s^2$ for all $s \in \mathbb{R}_{0,+}$.
	\end{assumption} 
	\cref{assumption_alpha_exp_stable} indicates the controlled system with full knowledge of $\bm{f}(\cdot)$ is exponentially stable \cite{khalil2015nonlinear}, which can be realized by common applied model-based non-linear control law such as feedback linearization \cite{dai2023can}, back-stepping \cite{jiao2022backstepping} and computed torque control \cite{yang2025safe}.
	
	Given the specific form in \cref{assumption_alpha_exp_stable}, the analytical expression of $S(t|s_0,\phi)$ in \eqref{eqn_dynamics_s} is written as
	\begin{align} \label{eqn_solution_S_exp_stable}
		S_{\exp}(t|s_0,\phi) = \mathrm{e}^{- \lambda t} s_0 + (1 - \mathrm{e}^{- \lambda t}) \lambda^{-1} \lambda_{\gamma} \phi^2.
	\end{align}
	With $S_{\exp}( \cdot | \cdot, \cdot)$, the upper bounds of the Lyapunov function, i.e., $\bar{V}(t)$ and the tracking error, i.e., $\bar{e}(\cdot)$, in \cref{theorem_control_performance,theorem_double_trigger} have closed-form expressions as
	\begin{align}
		\bar{V}(t) &\!=\! \mathrm{e}^{- \lambda t} V(0) + (1 - \mathrm{e}^{- \lambda t}) \lambda^{-1} \lambda_{\gamma} \underline{\phi}^2, \\
		\bar{e}(t) &\!=\! \lambda_1^{-1 / 2} \bar{V}^{1 / 2}(t) \!\le\! (\lambda_2 / \lambda_1)^{\frac{1}{2}} \mathrm{e}^{- \frac{\lambda}{2} t} e(0) \!+\! ( \lambda_{\gamma} / (\lambda \lambda_1) )^{\frac{1}{2}} \underline{\phi}, \nonumber
	\end{align}
	resulting in input-to-state stability of the tracking error \cite{khalil2015nonlinear}.
	
	Moreover, with the closed form $S_{\exp}(\cdot| \cdot, \cdot)$ the functions $\rho^{\rm{f}}(\cdot)$ and $\rho^{\rm{u}}(\cdot)$ have analytical expressions.
	Specifically, with \cref{assumption_alpha_exp_stable} the event trigger function $\rho^{\rm{u}}(\cdot)$ for control update in \eqref{eqn_rho_u} is written as
	\begin{align} \label{eqn_rho_u_exp_stab}
		\rho^{\rm{u}}_{\exp}(t) = V(t) - S_{\exp}(t - t^{\rm{f}}_p | V(t^{\rm{f}}_p), \phi_{p,1} ).
	\end{align}
	The trigger function $\rho^{\rm{f}}(\cdot)$ in \eqref{eqn_rho_f} for model update becomes
	\begin{align} \label{eqn_rho_f_exp_stab}
		&\rho^{\rm{f}}_{\exp}(t^{\rm{f}}_p) = - S_{\exp}(\hat{\Delta}_{p,2} | \tilde{V}_{p,1}(\hat{\Delta}_{p,1}), \underline{\phi} - \phi_{p,2}^-),
	\end{align}
	where $\phi_{p,2}^- = \phi(\eta^-(\bm{x}(t^{\rm{f}}_p)), \Delta(t^{\rm{f}}_p) + \bar{\Delta}^+, \varepsilon) $ and
	\begin{align} \label{eqn_tilde_e_1_plus}
		\tilde{V}_{p,1}(\hat{\Delta}_{p,1}) = S_{\exp}(\hat{\Delta}_{p,1} | \bar{V}(t^{\rm{f}}_p) - V(t^{\rm{f}}_p), \underline{\phi} - \phi_{p,1}).
	\end{align}
	The variables $\hat{\Delta}_{p,1}$ and $\hat{\Delta}_{p,2}$ in \eqref{eqn_rho_f_exp_stab} are defined as
	\begin{align} \label{eqn_hat_Delta} 
		\hat{\Delta}_{p,1} \!\!=\!\! \begin{cases}
			0 & \!\!\!\!\!\!, \text{if}~ \Xi_{p,1} \!\!>\! 0 \\
			\bar{\Delta}^{\rm{p}} & \!\!\!\!\!\!, \text{otherwise}
		\end{cases}, && 
		\hat{\Delta}_{p,2} \!\!=\!\! \begin{cases}
			0 & \!\!\!\!\!\!, \text{if}~ \Xi_{p,2} \!\!>\! 0 \\
			\bar{\Delta}^+ & \!\!\!\!\!\!, \text{otherwise}
		\end{cases}, 
	\end{align}
	where $\Xi_{p,1}$ and $\Xi_{p,2}$ are defined with the information at $t^{\rm{f}}_p$ as
	\begin{align}
		\Xi_{p,1} =& \lambda (\bar{V}(t^{\rm{f}}_p) - V(t^{\rm{f}}_p)) - \lambda_{\gamma} (\underline{\phi}^2 - \phi_{p,1}^2 ), \label{eqn_theta_1} \\
		\Xi_{p,2} =& \lambda \tilde{V}_{p,1} - \lambda_{\gamma} ( \underline{\phi}^2 - (\phi_{p,2}^-)^2 ). \label{eqn_theta_2}
	\end{align}
	The event-triggered approach for exponentially stable nominal control laws is summarized in \cref{algorithm_double_trigger}.
	
	\begin{algorithm} [t]
		\caption{Event-trigger Mechanism}
		\label{algorithm_double_trigger}
		\begin{algorithmic} [1]
			\Statex \hspace*{-\algorithmicindent} \textbf{I. Event-triggered Online Learning}
			\Statex \textbf{Input:} $\mathbb{D}(t^{\rm{f}}_{p\!-\!1}),\! \bm{x}(t^{\rm{f}}_{p\!-\!1}),\! \bm{\mu}(\bm{x}(t^{\rm{f}}_{p\!-\!1})),\! \eta(\bm{x}(t^{\rm{f}}_{p\!-\!1}))$; $t^{\rm{f}}_p, \bm{x}(t^{\rm{f}}_p)$;
			\State $\Xi_{p,1} \leftarrow$ \eqref{eqn_theta_1}; $\hat{\Delta}_{p,1} \leftarrow$ \eqref{eqn_hat_Delta}; $\tilde{V}_{p,1}(\hat{\Delta}_{p,1}) \leftarrow$ \eqref{eqn_tilde_e_1_plus};
			\State $\bm{\mu}(\bm{x}(t^{\rm{f}}_p)), \bm{\Sigma}(\bm{x}(t^{\rm{f}}_p)), \eta(\bm{x}(t^{\rm{f}}_p)) \!\leftarrow\!$ \eqref{eqn_GP_prediction}, \eqref{eqn_GP_error_bound} with $\mathbb{D}(t^{\rm{f}}_{p-1})$;
			\State $\Xi_{p,2} \leftarrow$ \eqref{eqn_theta_2}; $\hat{\Delta}_{p,2} \leftarrow$ \eqref{eqn_hat_Delta}; $\rho^{\rm{f}}_{\exp}(t^{\rm{f}}_p) \!\leftarrow\!$ \eqref{eqn_rho_f_exp_stab};
			\If {$\rho^{\rm{f}}_{\exp}(t^{\rm{f}}_p) > 0$}
			\State Update GP model with $\mathbb{D}(t^{\rm{f}}_p) \leftarrow$ \eqref{eqn_dataset_ET_update}; 
			\State Update $\bm{\mu}(\bm{x}(t^{\rm{f}}_p)), \!\bm{\Sigma}(\bm{x}(t^{\rm{f}}_p)) \!\leftarrow\!$ \eqref{eqn_GP_prediction} with $\mathbb{D}(t^{\rm{f}}_p)$;
			\State Update $\eta(\bm{x}(t^{\rm{f}}_p)) \leftarrow$ \eqref{eqn_GP_error_bound} with $\mathbb{D}(t^{\rm{f}}_p)$;
			\EndIf
			\State Send $\bm{\mu}(\bm{x}(t^{\rm{f}}_p))$ and $\eta(\bm{x}(t^{\rm{f}}_p))$ to event-triggered control;
			\State Save $\bm{x}(t^{\rm{f}}_p), \bm{\mu}(\bm{x}(t^{\rm{f}}_p)), \eta(\bm{x}(t^{\rm{f}}_p))$.
			\vspace*{-0.2cm}
			\Statex \hspace*{\dimexpr-\algorithmicindent-2pt\relax}\rule{0.48\textwidth}{0.4pt}
			\vspace*{-0.1cm}
		\end{algorithmic}
		\begin{algorithmic} [1]
			\Statex \hspace*{-\algorithmicindent} \textbf{II. Event-triggered Control}
			\State $\rho^{\rm{u}}_{\exp}(t) \leftarrow$ \eqref{eqn_rho_u_exp_stab} for current time $t$;
			\If {$\rho^{\rm{u}}_{\exp}(t) > 0$}
			\State Calculate $\bm{\pi}(t, \bm{x}(t), \bm{\mu}(\bm{x}(t^{\rm{f}}_{p-1})))$ and send to system;
			\EndIf
			\If {$t = t^{\rm{f}}_p$}
			\State Receive $\bm{\mu}(\bm{x}(t^{\rm{f}}_{p-1}))$ and $\eta(\bm{x}(t^{\rm{f}}_{p-1}))$ from GP;
			\State Save $t^{\rm{f}}_p$, $\bar{V}(t^{\rm{f}}_p)$, $V(t^{\rm{f}}_p)$ and $\phi_{p,1}$;
			\EndIf
		\end{algorithmic}
	\end{algorithm}
	
	\section{Simulations on Franka Emika Panda}
	\label{section_simulation}
	
	\subsection{Simulation Setting}
	
	To show the effectiveness of the proposed method in practical scenarios, simulations on a commercial manipulator with $7$ degrees of freedom (DoFs), namely Franka Emika Panda, is considered.
	The dynamics of Franka Emika Panda in joint space is modeled following Euler-Lagrange system as
	\begin{align} \label{eqn_Franka_EL_dynamics}
		\bm{M}(\bm{q}) + \bm{f}_q(\bm{x}) +\bm{u} = \bm{0}_{7 \times 1},
	\end{align}
	where the generalized coordinate $\bm{q} \in \mathbb{R}^7$ is defined as the rotation angle of each joint and the system state denotes $\bm{x} = [\bm{q}^T, \dot{\bm{q}}^T] \in \mathbb{X} \subset \mathbb{R}^{14}$.
	The state domain $\mathbb{X}$ is determined by the physical constraints for joint position and velocity, such that $\mathbb{X}$ is compact and known.
	The control input $\bm{u} \in \mathbb{R}^7$ represents the torque applied on each joint.
	The mass matrix $\bm{M}(\cdot): \mathbb{R}^7 \to \mathbb{R}^{7 \times 7}$ reflects the inertial property and is assumed as known.
	The force vector $\bm{f}_q(\cdot): \mathbb{R}^{14} \to \mathbb{R}^7$, including Coriolis force, gravity and joint friction, is considered as unknown.
	The dynamical system \eqref{eqn_Franka_EL_dynamics} is also written as \eqref{eqn_general_system} with
	\begin{align}
		&\bm{h}(\bm{x}, \bm{u}) = [\dot{\bm{q}}^T, - \bm{u}^T \bm{M}^{-T}(\bm{q})]^T, \\
		&\bm{f}(\bm{x}) = [\bm{0}_{1 \times 7}, - \bm{f}_q^T(\bm{x}) \bm{M}^{-T}(\bm{q})]^T,
	\end{align}
	where the first $7$ dimension of $\bm{f}(\cdot)$ is known to be $0$.
	Gaussian processes are used to model the last $7$ dimension of $\bm{f}(\cdot)$, i.e., $- \bm{M}^{-1}(\bm{q}) \bm{f}_q(\bm{x})$, where the squared exponential kernel with automatic relevance determination (ARD-SE kernel) is chosen for each dimension $f_i(\cdot)$ with $i = 1, \cdots, 7$ in the form of
	\begin{align} \label{eqn_Franka_ARDSE_kernel}
		\kappa_i(\bm{x}, \bm{x}') = \sigma_{f,i}^2 \exp \Big( -\frac{1}{2} \sum\nolimits_{j = 1}^{14} \frac{(x_j - x'_j) ^ 2}{ \sigma_{i,j}^2 } \Big)
	\end{align}
	where $\sigma_{f,i} \in \mathbb{R}_{0,+}$ and $\sigma_{i,j} \in \mathbb{R}_{0,+}$ with $i = 1, \cdots, 7$, $j = 1, \cdots, 14$ are obtained from hyperparameter optimization.
	The data set $\mathbb{D}(\cdot)$ used for GP regression satisfies \cref{assumption_dataset}, where the variances $\sigma_{w,i}^2$ with $i = 1, \cdots, 7$ of measurement noise are also from hyperparameter optimization.
	The computation time is modeled as $\Delta^{\rm{p}}(t) = \Delta^{\rm{u}}(t) = c_{\Delta} | \mathbb{D}(t) |^2$ with $c_{\Delta} = 2.5 \times 10^{-7}$, such that the maximal prediction and update time denote $\bar{\Delta}^{\rm{p}} = \bar{\Delta}^{\rm{u}} = 0.01$ with bounded size of data set $| \mathbb{D}(t) | \le 200$ for any $t \in \mathbb{R}_{0,+}$.

	The control task is to track a joint reference trajectory $\bm{x}_r(\cdot): \mathbb{R}_{0,+} \to \mathbb{X}$, which is composed of a reference joint position $\bm{q}_r(\cdot): \mathbb{R}_{0,+} \to \mathbb{R}^7$ and velocity $\dot{\bm{q}}_r(\cdot)$.
	The desired $\bm{q}_r(\cdot)$ and $\dot{\bm{q}}_r(\cdot)$ are obtained by inverse kinematic from the reference of end effector $\bm{x}_{r,e}(t) = [\bm{r}_e^T(t), \bm{\varphi}_e^T(t)]^T$, where
	\begin{align}
		\bm{r}_e(t) = \begin{bmatrix}
			0.35 + 0.1 \sin(2t) \\ 
			0.1 \cos(2t) \\
			0.6 + 0.01 \sin(20t)
		\end{bmatrix}, &&
		\bm{\varphi}_e(t) = \begin{bmatrix}
			\pi \\ 0 \\ 0
		\end{bmatrix}
	\end{align}
	with Tait–Bryan angle following $x \!-\! y \!-\! z$ sequence.
	The reference trajectory is shown in \cref{figure_Franka_reference}.
	To track the reference joint trajectory $\bm{x}_r(\cdot)$, the control law $\bm{\pi}(\cdot, \cdot, \cdot)$ is designed as	
	\begin{align}
		\bm{\pi}(t, \bm{x}, \bm{f}_x) = - \bm{M}(\bm{q}) ( \bm{K} (\bm{x} - \bm{x}_r(t)) + \ddot{\bm{q}}_r(t) ) - \bm{f}_x,
	\end{align}
	where the control gain $\bm{K} \in \mathbb{R}_{14 \times 7}$ is chosen as $\bm{K} = [-100 \bm{I}_7, -50 \bm{I}_7]$.
	The Lyapunov function is selected as $V(t, \bm{x}(t) - \bm{x}_r(t)) = (\bm{x}(t) - \bm{x}_r(t))^T \bm{P} (\bm{x}(t) - \bm{x}_r(t))$, where $\bm{P} \in \mathbb{R}^{14 \times 14}$ is the unique positive definite solution of Lyapunov equation $\bm{A}^T \bm{P} + \bm{P} \bm{A} = - \bm{I}_{14}$ due to Hurwitz
	\begin{align}
		\bm{A} = \begin{bmatrix}
			\bm{0}_{7 \times 7} & \bm{I}_7 \\
			-100 \bm{I}_7 & -50 \bm{I}_7
		\end{bmatrix}.
	\end{align}
	It is easy to see \cref{assumption_alpha_exp_stable} is satisfied with $\lambda_1 = \underline{\lambda}(\bm{P})$ and $\lambda_2 = \bar{\lambda}(\bm{P})$, where $\underline{\lambda}(\cdot)$ and $\bar{\lambda}(\cdot)$ return the minimal and maximal singular value of a square matrix.
	Moreover, $\lambda$ and $\lambda_{\gamma}$ are written as $\lambda = 1 / (2 \lambda_2)$ and $\lambda_{\gamma} = 2 \| \bm{P} \|^2$.
	The control trigger interval is set as $\varepsilon = 5 \times 10^{-3}$.
	The test is conducted by joint simulation with MatLab and CoppeliaSim \cite{coppeliaSim}, which provides a more practical simulation environment for robotics.
	
	\begin{figure}[t] 
		\centering
		\includegraphics[height=4cm]{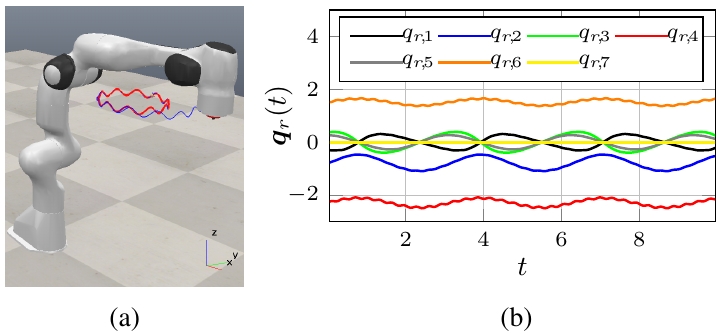}
		\caption{
			a) Franka Emika Panda in its initial position.
			The {\color{red} red} line represents the reference trajectory of the end effector, and the {\color{blue} blue} line is generated by online learning-based control with the proposed asynchronous event-triggered mechanism.
			b) The desired joint reference trajectory $\bm{q}_r = [q_{r,1}, q_{r,2}, q_{r,3}, q_{r,4}, q_{r,5}, q_{r,6}, q_{r,7}]^T$ over time $t$, which is generalized by inverse kinematics.
		}
		\label{figure_Franka_reference}
	\end{figure}
	
	\subsection{Control and Learning Performance}
	
	To demonstrate the effectiveness of the proposed event-trigger in \cref{section_double_trigger}, the following online learning-based control methods are used as baseline for comparison.
	\begin{enumerate}
		\item Time-triggered control with time-triggered online learning (Time Trigger):
		The control inputs are updated at $\{ t^{\rm{u}}_k \}_{k \in \mathbb{N}}$ with $t^{\rm{u}}_{k+1} \!-\! t^{\rm{u}}_k \!=\! \varepsilon$ for all $k \!\in\! \mathbb{N}$, and each pair $\{ \bm{x}(t^{\rm{f}}_p), \bm{y}(t^{\rm{f}}_p) \}$ is added to the data set at $t^{\rm{f}}_p$ for all $p \!\in\! \mathbb{N}$.
		\item Event-triggered control with time-triggered online learning (Control Trigger):
		Apply the event-trigger strategy for control update with $\rho^{\rm{u}}(\cdot)$ in \eqref{eqn_rho_u_exp_stab}, and add each data pair $\{ \bm{x}(t^{\rm{f}}_p), \bm{y}(t^{\rm{f}}_p) \}$ to $\mathbb{D}(t^{\rm{f}}_p)$ for all $k \in \mathbb{N}$.
		\item Time-triggered control with event-triggered online learning (GP Trigger):
		Transmit each control command to the controller via communication channel, and apply the event-triggered online learning with $\rho^{\rm{f}}(\cdot)$ in \eqref{eqn_rho_f_exp_stab}.
		\item Proposed asynchronous event-triggered strategy (Double Trigger):
		The control inputs $\bm{u}(\cdot)$ and GP model in terms of $\mathbb{D}(\cdot)$ are updated following event-triggered mechanism with $\rho^{\rm{u}}(\cdot)$ in \eqref{eqn_rho_u_exp_stab} and $\rho^{\rm{f}}(\cdot)$ in \eqref{eqn_rho_f_exp_stab}.
	\end{enumerate}
	The total simulation time is $10$, and the robot is initialized with $\bm{q}(0) = [0,0,0,-\pi/2,0,\pi/2,0]^T$ and $\dot{\bm{q}}(0) = \bm{0}_{7 \times 1}$.
	
	The control performance reflected by the value of Lyapunov function $V(\cdot)$ and the tracking error $e(\cdot)$ is shown in \cref{figure_Franka_performance_OneComparison}.
	It is obvious to see both $V(\cdot)$ and $e(\cdot)$ are similar among these $4$ methods, which are bounded by $\bar{V}(\cdot)$ and $\bar{e}(\cdot)$ respectively.
	Moreover, the prediction performance using the proposed asynchronous event-trigger is shown in \cref{figure_prediction}.
	The upper part of \cref{figure_prediction} shows that the prediction $\bm{\mu}(\cdot)$ from GP is close to the true unknown function $\bm{f}(\cdot)$, and the prediction error in the bottom part of \cref{figure_prediction} is bounded as shown in \cref{lemma_GP_error_bound}.
	Next, the communication and computation efficiency is investigated, which is reflected by the number of control and online learning trigger events shown in \cref{figure_trigger_instance_Franka}.
	Totally, there are $190$ and $192$ data pairs collected by event-triggered online learning in 3) and 4) among $1182$ and $1179$ prediction times, respectively.
	This observation shows comparable online learning efficiency between 3) and 4), although the proposed event-trigger saves slightly less computational power with $16.29\%$ triggers than 3) with $16.07\%$.
	Considering the trigger event times for control, both 2) and 4) show high efficiency in saving communication resources by triggering $1307$ and $1251$ times among $2000$ potential communication instances.
	Overall, the proposed asynchronous event-trigger achieves similar performance with high efficiencies in both communication and computation.
	
	\begin{figure}[t] 
		\centering
		\includegraphics[width=0.48\textwidth]{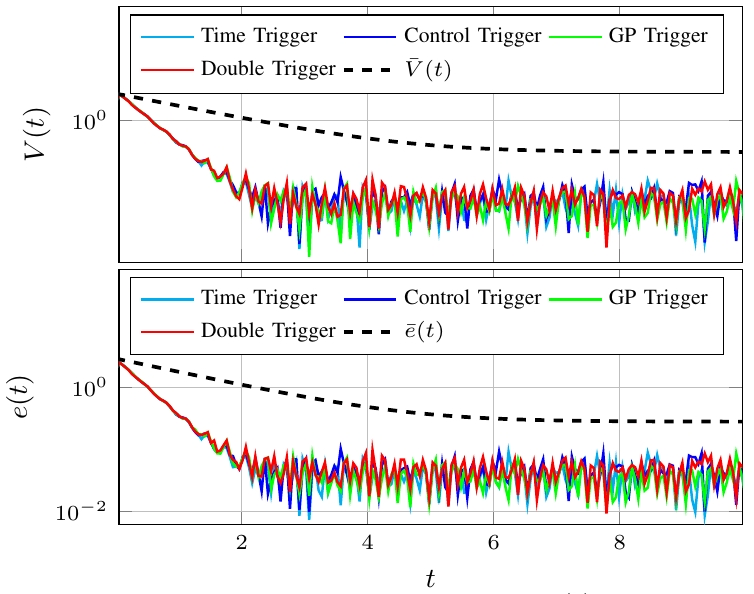}
		\caption{
			The value of Lyapunov function $V(t)$ and tracking error $e(t)$ w.r.t time $t$ from $4$ learning-based control methods.
		}
		\label{figure_Franka_performance_OneComparison}
	\end{figure}
	
	\begin{figure}[t] 
		\centering
		\includegraphics[width=0.48\textwidth]{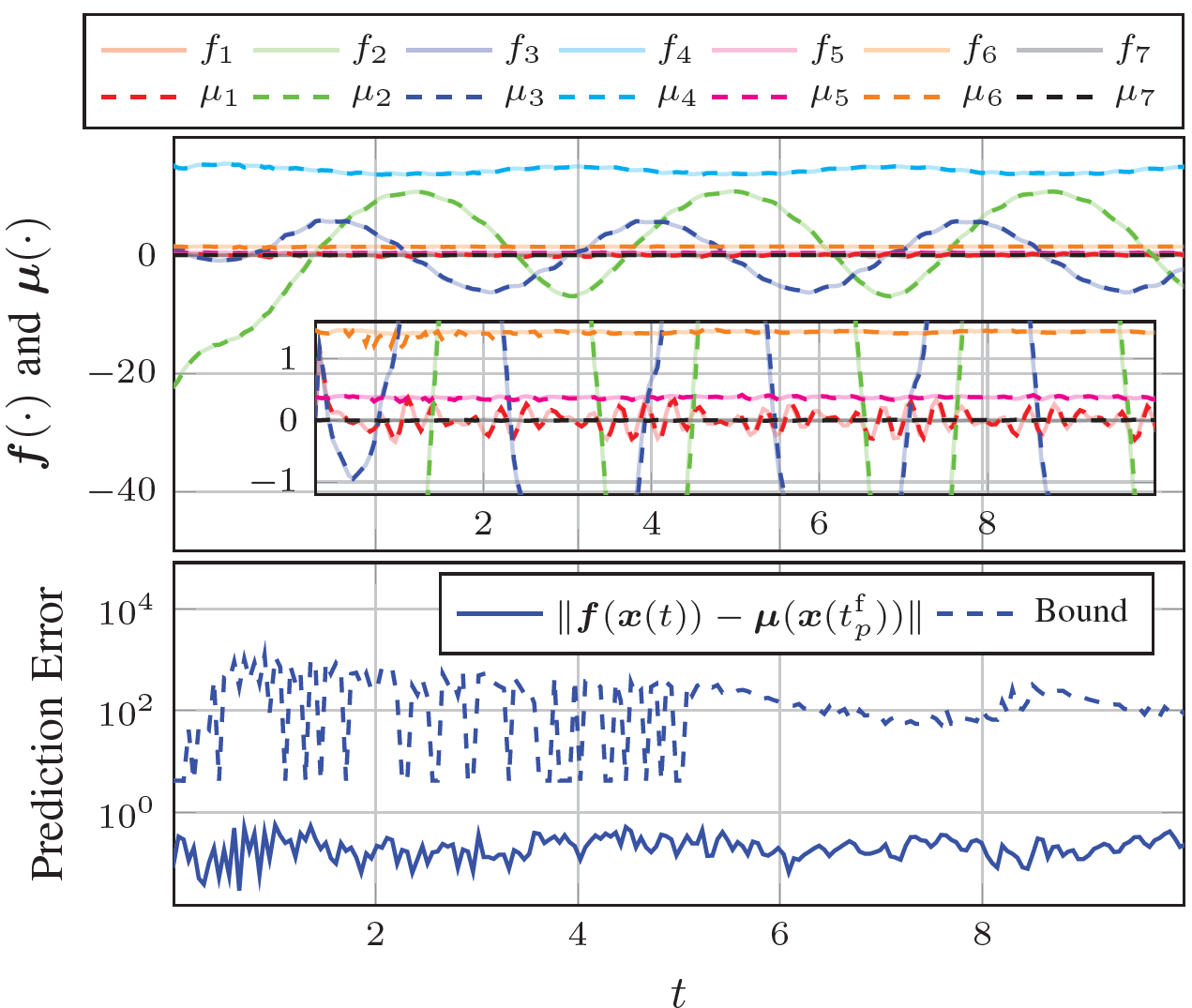}
		\caption{
			Prediction performance by using Double Trigger.
		}
		\label{figure_prediction}
	\end{figure}
	
	\begin{figure}[t] 
		\centering
		\begin{subfigure}{0.5\textwidth}
			\centering
			\includegraphics[width=0.98\textwidth]{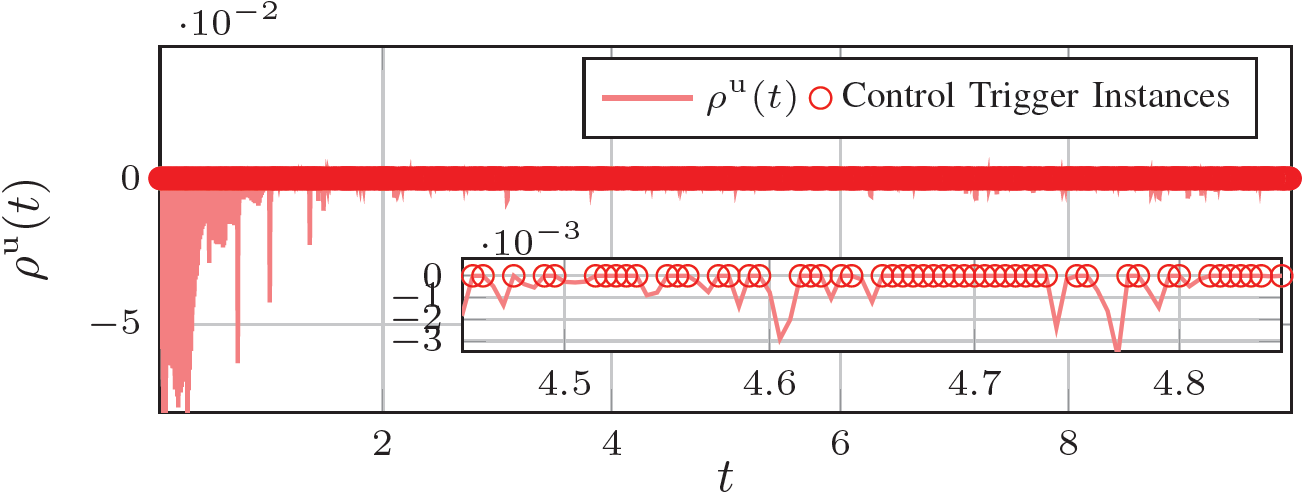}
			\caption{
				Control Trigger for Franka Emika Panda.
			}
		\end{subfigure}
		\begin{subfigure}{0.5\textwidth}
			\centering
			\includegraphics[width=0.98\textwidth]{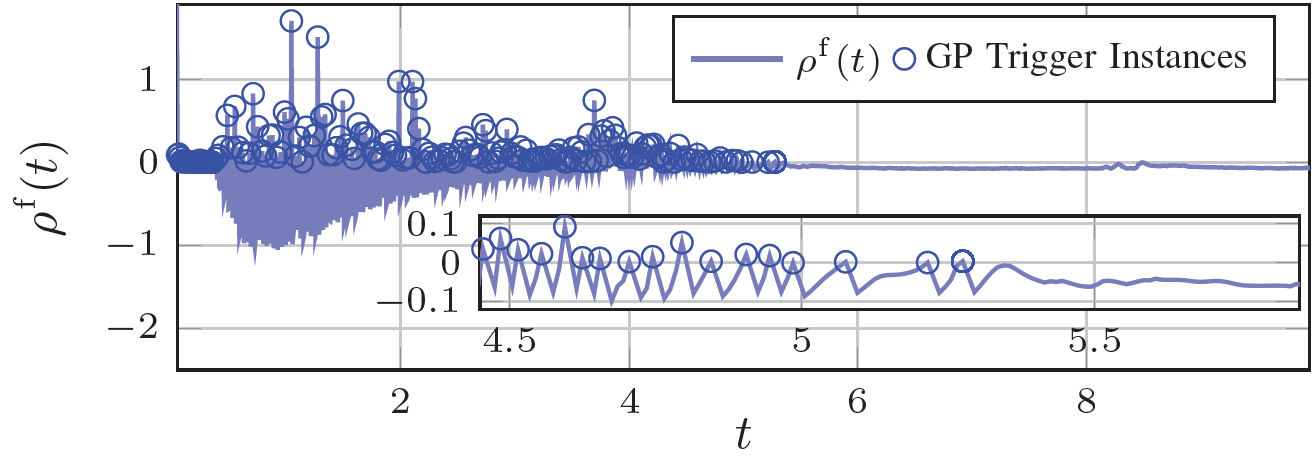}
			\caption{
				GP Trigger for Franka Emika Panda.
			}
		\end{subfigure}
		\begin{subfigure}{0.5\textwidth}
			\centering
			\includegraphics[width=0.98\textwidth]{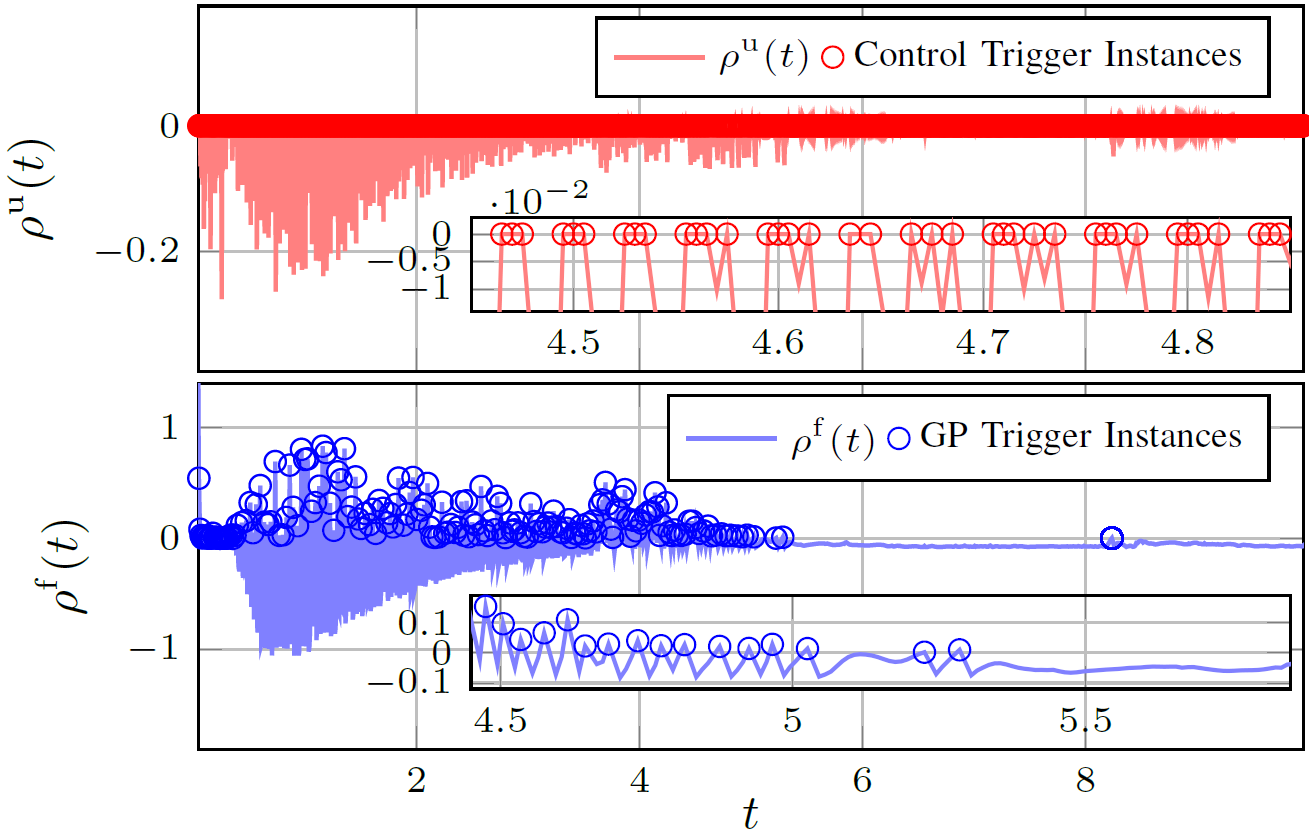}
			\caption{
				Double Trigger for Franka Emika Panda.
			}
		\end{subfigure}
		\caption{
			Value of trigger function $\rho^{\rm{u}}(t)$ and $\rho^{\rm{f}}(t)$ over time $t$ and trigger instances for control update and online learning, respectively.
			It is obvious to see $\rho^{\rm{u}}(t) \le 0$ for any $t \in \mathbb{R}_{0,+}$, indicating $V(t) \le \hat{V}(t)$.
			This also means the control triggers occur at $t$ when $\rho^{\rm{u}}(t) = 0$, i.e., $V(t) = \hat{V}(t)$.
			Due to the existence of computational delay, the value of $\rho^{\rm{f}}(\cdot)$ may greater than $0$.
			However, the online learning triggered when $\rho^{\rm{f}}(\cdot) \ge 0$ ensures $\rho^{\rm{f}}(\cdot) < 0$ after GP model update.
		}
		\label{figure_trigger_instance_Franka}
	\end{figure}
	
	\section{Conclusions}
	\label{section_conclusion}
	
	In this paper, we consider a control problem for unknown systems with online machine learning, where learning-induced computational delay is non-negligible.
	To mitigate the delay effect, an in-network online learning-based control structure is employed.
	The control performance reflected by the tracking error bound is shown under computational delays, allowing for different control and online learning strategy satisfying a specific condition.
	Moreover, a trade-off between communication and computation is shown for a desired performance, and an exemplary strategy with time-trigger is shown to achieve the required condition.
	Furthermore, an asynchronous event-trigger framework is proposed, ensuring the desired performance with higher efficiencies in both communication and computation without the exhibition of Zeno behavior.
	Additionally, the closed form of the trigger conditions is derived for exponentially stabilizable systems, showing that the proposed event-trigger is practically implementable.
	Finally, the effectiveness of the proposed asynchronous event-trigger framework is demonstrated through simulations.
	
	While the proposed method ensures the control performance efficiently, the conservatism in the trigger design induces more as ideal case, which mainly arises from two sources.
	First, the expected control performance $\bar{V}(\cdot)$ is derived under worst-case assumptions on the prediction and delay-induced errors.
	Specifically, the prediction error bound $\underline{\eta}$ considers the single-point GP accuracy, while the delay-induced error term $F\bar{\Delta}$ is based on the largest possible state derivative $\|\dot{\bm{x}}(t)\| \le F$.
	Both factors overestimate the actual uncertainty and lead to conservative triggering thresholds.
	Second, because the computation delay introduces a two-phase effect covering both the computation and utilization intervals, the trigger condition must rely on the worst-case evolution of the Lyapunov function $V(\cdot)$, which further increases conservatism.
	
	Reducing this conservatism is an important direction for future research.
	Possible approaches include using multi-point prediction accuracy bound for $\underline{\eta}$ and using more accurate estimation of $V(\cdot)$ via local Lipschitz constants.
	
	\bibliographystyle{IEEEtran}
	\bibliography{refs_revised}
	
	\def\BiographyDistance{-10pt}
	\vspace{-11pt}
	\begin{IEEEbiography}[{\includegraphics[width=1in,height=1.25in,clip,keepaspectratio]{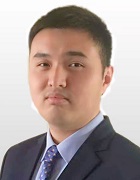}}]{Xiaobing Dai}
		received the B.Sc. mechanical engineering from the Tongji University, Shanghai, China, in 2018 with direction in mechatronics, building environment and civil engineering. 
		He received double M.Sc degrees in Mechanical Engineering, Mechatronics and Robotics from the Technical University of Munich, Munich, Germany, in 2021. Since February 2022, he is a PhD student at the Chair of Information-oriented Control, TUM School of Computation, Information and Technology at the Technical University of Munich, Munich, Germany. His current research interests include efficient online machine learning, networked control systems, safe learning-based control.
	\end{IEEEbiography}
	\vspace{\BiographyDistance}
	
	\begin{IEEEbiography}[{\includegraphics[width=1in,height=1.25in,clip,keepaspectratio]{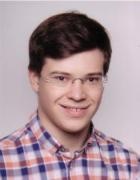}}]{Armin Lederer}
		(Member, IEEE) 
		received the B.Sc. and M.Sc. degree in electrical engineering and information technology from the Technical University of Munich, Munich, Germany, in 2015 and 2018, respectively. From 2018 until 2023, he has been a PhD student at the Chair of Information oriented Control, Department of Electrical and Computer Engineering at the Technical University of Munich, Germany. Since 2023 he is a postdoctoral researcher in the Learning \& Adaptive Systems Group at ETH Zurich, Switzerland. His current research interests include the stability of data-driven control systems and machine learning in closed-loop systems.
	\end{IEEEbiography}
	\vspace{\BiographyDistance}
	
	\begin{IEEEbiography}[{\includegraphics[width=1in,height=1.25in,clip,keepaspectratio]{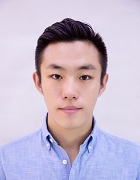}}]{Zewen Yang}
		(Member, IEEE)
		received the M.S. degree in control engineering from Northeast Forest University, in 2017. He pursued a Ph.D. in control science and engineering at College of Intelligent Systems Science and Engineering, Harbin Engineering University, Harbin, China, from 2017 to 2019 and joined the Chair of Information-oriented Control, School of Computation, Information and Technology, the Technical University of Munich, Munich, Germany, in machine learning and data-driven control until 2023. His current research interests include multi-agent systems, cooperative learning, control theory, and general robotics.
	\end{IEEEbiography}
	\vspace{\BiographyDistance}
	
	\begin{IEEEbiography}[{\includegraphics[width=1in,height=1.25in,clip,keepaspectratio]{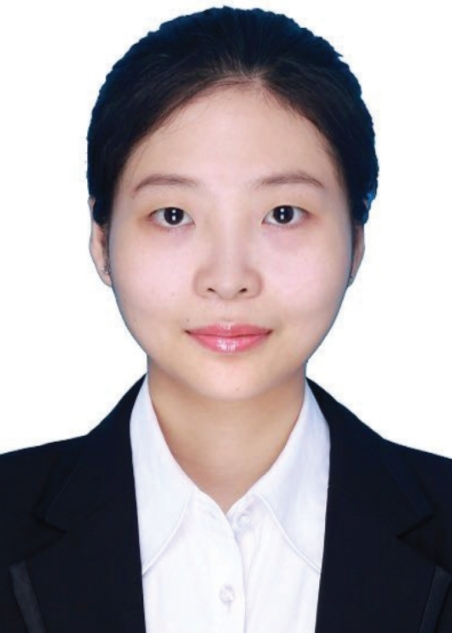}}]{Sihua Zhang}
		received the M.Eng. degree in control engineering from the Beijing Institute of Technology, Beijing, China, in 2021. She is currently pursuing the doctor's degree in control engineering with the School of Automation, Beijing Institute of Technology, Beijing, China. From 2023 to 2025, she is additionally working as research assistant at the Chair of Information-oriented Control, TUM School of Computation, Information and Technology at the Technical University of Munich, Germany. Her current research interests include safety-critical robotic control, control barrier functions, machine learning and optimal control.
	\end{IEEEbiography}
	\vspace{\BiographyDistance}
	
	\begin{IEEEbiography}[{\includegraphics[width=1in,height=1.25in,clip,keepaspectratio]{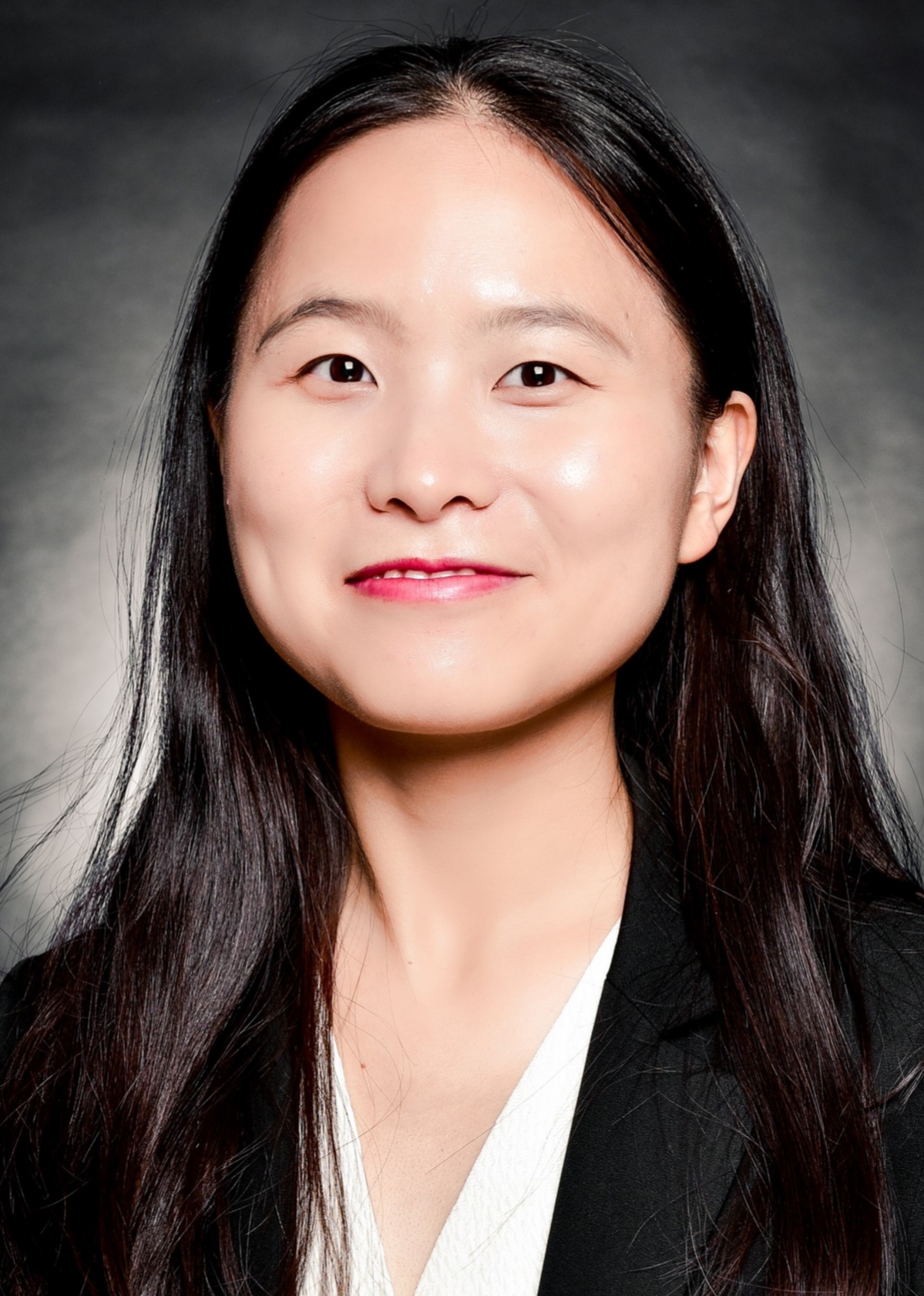}}]{Lu Wan}
		received a B.Sc. in Built Environment and Energy Engineering from Tongji University, China, in 2018, and an M.Sc. in Energy Engineering from the University of Stuttgart, Germany, in 2021. She is currently a corporate researcher at Robert Bosch GmbH, Germany, and pursuing a Ph.D. with the Information Systems in the Built Environment research group at Eindhoven University of Technology. Her research interests include building performance simulation, dynamic modelling of building and HVAC systems, model predictive control for energy systems, IoT, building information modelling, semantic web technologies, and digital twins.
	\end{IEEEbiography}
	\vspace{\BiographyDistance}
	\vfill
	
	\begin{IEEEbiography}[{\includegraphics[width=1in,height=1.25in,clip,keepaspectratio]{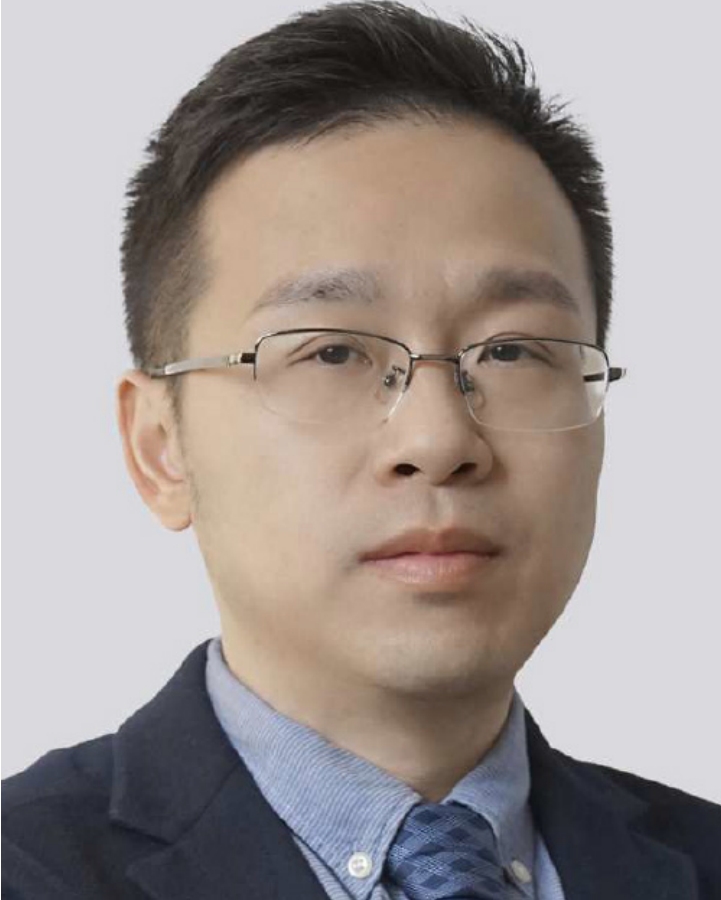}}]{Yang Tang}
		(Fellow, IEEE)
		received the B.S. and Ph.D. degrees in electrical engineering from Donghua University, Shanghai, China, in 2006 and 2010, respectively.
		From 2008 to 2010, he was a Research Associate with The Hong Kong Polytechnic University, Hong Kong. From 2011 to 2015, he was a Postdoctoral
		Researcher with the Humboldt University of Berlin, Berlin, Germany, and the Potsdam Institute for Climate Impact Research, Potsdam, Germany. He is currently a Professor with the East China University of Science and Technology, Shanghai. He has published more than 200 papers in international journals and conferences, including more than 140 papers in IEEE TRANSACTIONS and 20 papers in IFAC journals. His current research interests include distributed estimation/control/optimization, computer vision, reinforcement learning, cyber–physical systems, hybrid dynamical systems, and their applications.
		
		Prof. Tang was a recipient of the Alexander von Humboldt Fellowship. He is a Senior Area Editor of IEEE Transactions on Circuits and Systems—Part I: Regular Papers and an Associate Editor of IEEE Transactions on Neural Networks and Learning Systems, IEEE Transactions on Cybernetics, IEEE Transactions on Industrial Informatics, IEEE/ASME Transactions on Mechatronics, IEEE Transactions on Cognitive and Developmental Systems, IEEE Transactions on Emerging Topics in Computational Intelligence, IEEE Systems Journal, Engineering Applications of Artificial Intelligence (IFAC), Science China Information Sciences, and Automatica Sinica. He has been awarded as the best/outstanding associate editor in IEEE journals for four times. He is a (leading) guest editor for several special issues focusing on autonomous systems, robotics, and industrial intelligence in IEEE Transactions.
	\end{IEEEbiography}
	\vspace{\BiographyDistance}
	
	\begin{IEEEbiography}[{\includegraphics[width=1in,height=1.25in,clip,keepaspectratio]{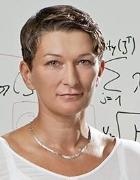}}]{Sandra Hirche}
		(Fellow, IEEE)
		received the Dipl.-Ing degree in aeronautical engineering from the Technical University of Berlin, Berlin, Germany, in 2002, and the Dr. Ing. degree in electrical engineering from the Technical University of Munich, Munich, Germany, in 2005. From 2005 to 2007, she was awarded a Post-doctoral scholarship from the Japanese Society for the Promotion of Science at the Fujita Laboratory, Tokyo Institute of Technology, Tokyo, Japan. From 2008 to 2012, she was an Associate Professor with the Technical University of Munich. Since 2013, she has served as Technical University of Munich Liesel Beckmann Distinguished Professor and has been with the Chair of Information-Oriented Control, Department of Electrical and Computer Engineering, Technical University of Munich. She has authored or coauthored more than 150 papers in international journals, books, and refereed conferences. Her main research interests include cooperative, distributed, and networked control with applications in human--machine interaction, multirobot systems, and general robotics. 
		
		Dr. Hirche has served on the editorial boards of the IEEE Transactions on Control of Network Systems, the IEEE Transactions on Control Systems Technology, and the IEEE Transactions on Haptics. She has received multiple awards such as the Rohde \& Schwarz Award for her Ph.D. thesis, the IFAC World Congress Best Poster Award in 2005, and -- together with students -- the 2018 Outstanding Student Paper Award of the IEEE Conference on Decision and Control as well as Best Paper Awards from IEEE Worldhaptics and the IFAC Conference of Manoeuvring and Control of Marine Craft in 2009.
	\end{IEEEbiography}
	\vspace{\BiographyDistance}
	
	\vfill
	
\end{document}